\documentclass[hidelinks,onefignum,onetabnum]{siamart251216}

\usepackage{etoolbox}

\usepackage{lipsum}
\usepackage{amsfonts}
\usepackage{graphicx}
\usepackage{epstopdf}
\usepackage{makecell}
\usepackage{algorithm}
\usepackage{algpseudocode}
\usepackage{microtype}
\usepackage{amssymb}
\usepackage{graphicx}
\usepackage[colorinlistoftodos]{todonotes}
\usepackage{hyperref}
\usepackage{mathrsfs}
\usepackage{latexsym}
\usepackage{amsmath}
\usepackage{appendix}
\usepackage[cal=cm]{mathalfa}
\usepackage{epstopdf}
\usepackage{wrapfig}
\usepackage{float}
\usepackage{color}
\usepackage[percent]{overpic}
\usepackage{threeparttable}
\usepackage{booktabs}
\usepackage{multirow}
\usepackage{caption}
\usepackage{subfigure}
\usepackage{tabularx}
\usepackage{todonotes}
\newtheorem{assumption}{Assumption}
\usepackage[numbers]{natbib}
\ifpdf
  \DeclareGraphicsExtensions{.eps,.pdf,.png,.jpg}
\else
  \DeclareGraphicsExtensions{.eps}
\fi

\newsiamremark{remark}{Remark}
\newsiamremark{hypothesis}{Hypothesis}
\crefname{hypothesis}{Hypothesis}{Hypotheses}
\newsiamthm{claim}{Claim}
\newsiamremark{fact}{Fact}
\crefname{fact}{Fact}{Facts}
\allowdisplaybreaks
\newcommand{\red}{\textcolor{black}}
\newcommand{\cred}{\color{black}}
\newcommand{\blue}{\textcolor{black}}
\newcommand{\purp}{\textcolor{black}}
\newcommand{\bfw}{\mathbf{w}}
\newcommand{\bfeta}{\mathbf{\eta}}
\newcommand{\E}{\mathbb{E}}
\newcommand{\x}{\mathbf{x}}
\newcommand{\y}{\mathbf{y}}
\newcommand{\mH}{\mathbf{H}}

\headers{LAE-EnKF for Nonlinear Data assimilation}{X. T. Tong, Y. Wang and L. Yan}

\title{Latent Auto-encoder Ensemble Kalman Filter for Nonlinear Data Assimilation\thanks{Submitted to the editors DATE.
\funding{This work was supported by the NSF of China (Nos. 92370126, 12171085) and the Jiangsu Provincial Scientific Research Center of Applied Mathematics (grant BK20233002).}}}

\author{Xin T. Tong \thanks{Department of Mathematics, National University of Singapore, 119076, Singapore (xin.t.tong@nus.edu.sg). }
\and Yanyan Wang\thanks{School of Mathematics, Southeast University, Nanjing 210096, China (yanyanwang@seu.edu.cn).}
\and Liang Yan\thanks{School of Mathematics, Southeast University, Nanjing 210096, China (yanliang@seu.edu.cn).  }
}

\usepackage{amsopn}

\ifpdf
\hypersetup{
  pdftitle={Latent Autoencoder Ensemble Kalman Filter for Nonlinear Data Assimilation},
  pdfauthor={X. T.  Tong, Y. Wang and L. Yan}
}
\fi

\begin{document}

\maketitle
\begin{abstract}
The ensemble Kalman filter (EnKF) is widely used for data assimilation in high-dimensional systems, but its performance often deteriorates for strongly nonlinear dynamics due to the structural mismatch between the Kalman update and the underlying system behavior. In this work, we propose a latent autoencoder ensemble Kalman filter (LAE-EnKF) that addresses this limitation by reformulating the assimilation problem in a learned latent space with linear and stable dynamics. The proposed method learns a nonlinear encoder--decoder together with a stable linear latent evolution operator and a consistent latent observation mapping, yielding a closed linear state-space model in the latent coordinates. This construction restores compatibility with the Kalman filtering framework and allows both forecast and analysis steps to be carried out entirely in the latent space. Compared with existing autoencoder-based and latent assimilation approaches that rely on unconstrained nonlinear latent dynamics, the proposed formulation emphasizes structural consistency, stability, and interpretability.
We provide a theoretical analysis of learning linear dynamics on low-dimensional manifolds and establish generalization error bounds for the proposed latent model. Numerical experiments on representative nonlinear and chaotic systems demonstrate that the LAE-EnKF yields more accurate and stable assimilation than the standard EnKF and related latent-space methods, while maintaining comparable computational cost and data-driven.
\end{abstract}

\begin{keywords}
Data Assimilation; Ensemble Kalman Filter;  Latent  Space.
 \end{keywords}

\begin{MSCcodes}
68T07,37M05,65P99, 65C20
\end{MSCcodes}

\section{Introduction}
Data assimilation (DA) optimally integrates dynamical model predictions with noisy and incomplete observations  to estimate the evolving state of complex systems. In practice, it combines heterogeneous data sources, such as in-situ measurements, satellite retrievals, and radar observations, with either physics-based or data-driven models of the underlying dynamics. By providing a systematic framework for uncertainty‑aware state estimation,  DA serves as a foundational methodology across numerous scientific and engineering fields, including numerical weather prediction, climate science, geophysics, hydrology, and economics~\cite{Kalnay2002}.

 From a probabilistic perspective, DA can be formulated as a Bayesian filtering problem~\cite{Reich2015probability}, where the goal is to sequentially infer the posterior distribution of the system state conditioned on all available observations. Among approximate Bayesian filtering methods, the ensemble Kalman filter (EnKF)~\cite{evensen2009data} has become one of the most widely adopted techniques due to its favorable computational scaling, derivative-free implementation, and ensemble-based representation of uncertainty. 
\red{In the linear--Gaussian setting, the EnKF recovers the classical Kalman filter, which provides the optimal linear update for Gaussian systems, in the large-ensemble limit~\cite{Mandel2011}. For nonlinear systems, however, the large-ensemble limit of the EnKF generally does not recover the true filtering distribution~\cite{Ernst2015}. This limitation has been investigated from various theoretical perspectives, including mean-field and non-asymptotic analyses, which provide insight into the behavior of ensemble-based approximations in nonlinear settings~\cite{carrillo2024,AlGhattas2023resampling,SanzAlonso2023,SanzAlonsoWaniorek2025}. Despite these theoretical insights, the practical performance of the EnKF remains fundamentally constrained by finite ensemble sizes. Sampling errors in the empirical covariance can induce spurious long-range correlations, significantly degrading estimation accuracy. To alleviate these issues, techniques such as localization and covariance inflation are commonly employed~\cite{Tong2016stable,Tong2016nonlinear}. However, these remedies address sampling errors but not the underlying limitation. In nonlinear, partially observed, or non-Gaussian settings, the EnKF can become suboptimal and even diverge~\cite{kelly2015concrete}.} This degradation is not intrinsic to the Kalman update itself, but rather stems from a structural mismatch between the linear-Gaussian assumptions underlying the EnKF analysis step and the true nonlinear dynamics and observation processes. As a result, the posterior ensemble is effectively constrained to an affine subspace determined by local linearizations of the model and observation operator~\cite{Reich2015probability,liu2025dropout}, leading to biased estimates and loss of stability.

Extensive research has aimed to enhance the robustness of EnKF-type methods in nonlinear and non-Gaussian settings. Sigma-point filters, including the unscented and cubature Kalman filters, propagate carefully chosen samples through nonlinear models to improve moment approximations~\cite{Julier1997UKF,Arasaratnam2009CKF}. Particle filters offer greater flexibility by representing the posterior with weighted ensembles~\cite{Reich2015probability}, but they typically suffer from severe weight degeneracy in high-dimensional systems. \blue{More recently, generative approaches based on score and diffusion models have been proposed to directly model complex posterior distributions~\cite{Bao2024EnSF,Li2025SOAD,Rozet2023SDA}.} Alternative approaches reinterpret the Bayesian update as a deterministic transport or flow of probability measures~\cite{Spantini2022OT}, while hybrid variational–ensemble methods incorporate gradient information to better handle nonlinear observation operators~\cite{Buehner2010Hybrid}. Despite these advances, achieving accurate, stable, and scalable filtering for strongly nonlinear, high-dimensional systems remains a fundamental challenge.

Recently, deep learning has emerged as a powerful tool for addressing nonlinear DA problems~\cite{Arcucci2021DDA,Bocquet2021enkf,BUIZZA2022DAML}. Neural networks have been used to learn surrogate dynamical models, approximate observation operators, and correct model error directly from data~\cite{WANG2025CPC, WANG2026JCP, CHENG2022surrogate, Zhu2019DDA}. 
While some approaches iteratively retrain networks during the assimilation cycle~\cite{BRAJARD2020DAMLL96}, such strategies are often impractical for real-time sequential filtering and may lack stability and interpretability. These limitations have motivated the development of \emph{latent data assimilation} methods, which seek to reduce dimensionality and nonlinearity by performing DA in a learned low-dimensional latent space. One class of methods constructs surrogate dynamics in latent coordinates and incorporates observational corrections during time integration~\cite{LI2024LAINR,Peyron2021LA}. A complementary class embeds both system states and observations into a latent space and performs filtering updates directly in latent  variables~\cite{AKBARI2023latent,Amendola2021LA,cheng2023generalised}. However, in the presence of heterogeneous or strongly nonlinear observations, separate encoders are often required for states and observations, leading to mismatched latent representations and inconsistencies between forecast and analysis steps.

To address these issues, Chen et al. \cite{CHENG2024MEDLA} proposed the multi-domain encoder-decoder latent assimilation (MEDLA) framework, which maps heterogeneous observations into a unified latent space using multiple encoders and decoders. While MEDLA improves flexibility under diverse observation settings, its latent dynamics are governed by unconstrained nonlinear neural operators, which can limit interpretability and long-term stability. In parallel, Koopman operator theory offers a complementary perspective by representing nonlinear dynamical systems through linear evolution in a suitably lifted feature space \cite{Brunton2022Koopman,liu2024estimate}. Koopman autoencoders operationalize this idea by learning encoder–decoder pairs such that latent variables evolve under a learned linear operator~\cite{Azencot2020KAE,nayak2025temporally}. Crucially, linear latent dynamics provide improved interpretability, numerical stability, and structural compatibility with Kalman filtering.

Motivated by these observations, we propose in this paper the \emph{latent autoencoder ensemble Kalman filter} (LAE-EnKF), which reformulates the DA problem through a structure-preserving latent dynamical representation. The proposed framework learns a nonlinear encoder–decoder pair that maps high-dimensional physical states onto a compact latent manifold, while explicitly enforcing stable linear latent dynamics. In contrast to generic autoencoder-based DA methods, all forecast and analysis operations are carried out directly in the learned latent coordinates, where the Kalman update is structurally justified. The approach is fully data-driven and applicable in settings where the governing equations are partially known or unavailable. To accommodate heterogeneous and partially observed measurement settings, we further introduce an observation encoder that embeds observations into the same latent coordinate system as the state variables. This design yields a unified latent state-space model in which both the system dynamics and the observation process are linear in the latent variables, eliminating inconsistencies arising from mismatched latent representations. Owing to this structure, the intrinsic latent dimension required for accurate assimilation is often significantly smaller than the ambient physical dimension, leading to improved robustness and computational efficiency.

\red{On the theoretical side, we analyze the proposed framework under a smooth manifold assumption and establish generalization error bounds for both the latent linear dynamics and the observation encoder, while explicitly accounting for the effects of dynamical perturbations and observation noise.} Extensive numerical experiments on representative nonlinear and chaotic systems show that the proposed LAE-EnKF consistently improves assimilation accuracy, stability, and robustness compared with standard EnKF and existing autoencoder-assisted approaches, while maintaining favorable computational efficiency.

The remainder of this paper is organized as follows. Section~\ref{s2} formulates the data assimilation problem and reviews the EnKF. Section~\ref{s3} introduces the LAE-EnKF framework, including the latent representation, observation embedding, and filtering algorithm. Section~\ref{s4} presents the theoretical analysis. Section~\ref{s5} reports numerical results on several nonlinear and chaotic systems. Finally, Section~\ref{s6} concludes the paper and discusses future research directions.

\section{Data Assimilation}\label{s2}
\red{In this section, we introduce the problem formulation and the EnKF framework. For ease of presentation, we adopt the following notation. Lower-case letters denote scalars, bold lower-case letters denote vectors, and upper-case letters denote matrices. Calligraphic letters represent manifolds, sets, and function classes, while script letters denote neural network mappings. We use $\widehat{\cdot}$ to denote empirical or learned estimators, and $\widetilde{\cdot}$ to denote latent variables or approximation functions arising in the theoretical analysis.}


\subsection{Problem formulation}
We consider a discrete-time dynamical system of the form
\begin{equation}\label{forward}
\mathbf{x}_{k}
= \mathbf{F}\!\left(\mathbf{x}_{k-1}\right) + \mathbf{w}_k, 
\qquad 
\mathbf{w}_k \sim \mathcal{N}(\mathbf{0}, \mathbf{Q}_k),
\end{equation}
where $\mathbf{x}_{k} \in \mathbb{R}^{D}$ denotes the system state at time $t_k$, $\mathbf{F}$ is a generally nonlinear evolution operator, and $\mathbf{w}_k$ represents model error.

In many practical applications, such as numerical weather prediction and ocean circulation modeling, the full state $\mathbf{x}_k$ is not directly observable~\cite{Reich2015probability}. Instead, at each time $t_k$, only partial and noisy observations are available:
\begin{equation}\label{obs}
\mathbf{y}_k = \mathcal{H}(\mathbf{x}_{k}) + \boldsymbol{\eta}_k,
\qquad 
\boldsymbol{\eta}_k \sim \mathcal{N}(\mathbf{0}, \mathbf{\Gamma}_k),
\end{equation}
where $\mathcal{H}$ is the observation operator and $\boldsymbol{\eta}_k$ denotes observation noise.

Given observations $Y_k = (\mathbf{y}_1, \ldots, \mathbf{y}_k)$, the objective of data assimilation is to approximate the filtering distribution $p(\mathbf{x}_k \mid Y_k)$. For nonlinear and high-dimensional systems, evaluating this posterior exactly is intractable~\cite{Reich2015probability}, motivating the development of approximate Bayesian filtering methods.

\subsection{Ensemble Kalman filter}\label{s2.2}

The ensemble Kalman filter (EnKF)~\cite{evensen2009data} is one of the most widely used approximate Bayesian filtering methods for large-scale systems. It represents uncertainty using a finite ensemble and propagates empirical means and covariances without requiring explicit  derivatives of the dynamical model. Owing to its computational efficiency and scalability, the EnKF has become a standard tool in high-dimensional data assimilation.

At each assimilation cycle, the EnKF alternates between a forecast step and an analysis step. We summarize the standard formulation below to highlight the structural assumptions underlying the method, which are central to the developments in this paper.

\textit{Initialization.}  
Given a prior mean $\mathbf{m}_0$ and covariance $\mathbf{P}_0$, an initial ensemble
\[
\widehat{\mathbf{x}}_0^{(1)}, \ldots, \widehat{\mathbf{x}}_0^{(N_e)}
\]
is drawn from the prior distribution. The forecast and analysis steps are then applied sequentially for $k = 1, \ldots, K$.

\textit{Forecast step.}  
Each ensemble member is propagated forward according to the nonlinear dynamical model~\eqref{forward}:
\begin{equation*}
\widehat{\mathbf{x}}_{k|k-1}^{(j)}  = \mathbf{F}\!\left(\widehat{\mathbf{x}}_{k-1}^{(j)}\right) + \mathbf{w}_k^{(j)}, \qquad j = 1, \ldots, N_e,
\end{equation*}
where $\widehat{\mathbf{x}}_{k|k-1}^{(j)}$ denotes the forecast of the $j$-th ensemble member at time $t_k$ conditioned on observations up to $t_{k-1}$. The corresponding predicted observations are
\[
\widehat{\mathbf{y}}_k^{(j)} = \mathcal{H}\!\left(\widehat{\mathbf{x}}_{k|k-1}^{(j)}\right).
\]
The sample means are computed as
\[
\overline{\mathbf{x}}_{k|k-1} 
= \frac{1}{N_e} \sum_{j=1}^{N_e} \widehat{\mathbf{x}}_{k|k-1}^{(j)}, 
\qquad
\overline{\mathbf{y}}_k 
= \frac{1}{N_e} \sum_{j=1}^{N_e} \widehat{\mathbf{y}}_k^{(j)}.
\]
The sample cross-covariance and observation covariance are given by
\begin{align*}
\widehat{\mathbf{P}}_{xy} &= \frac{1}{N_e - 1} \sum_{j=1}^{N_e} 
(\widehat{\mathbf{x}}_{k|k-1}^{(j)} - \overline{\mathbf{x}}_{k|k-1}) (\widehat{\mathbf{y}}_k^{(j)} - \overline{\mathbf{y}}_k)^{\!\top},\\  
\widehat{\mathbf{P}}_{yy}  &= \frac{1}{N_e - 1} \sum_{j=1}^{N_e}  (\widehat{\mathbf{y}}_k^{(j)} - \overline{\mathbf{y}}_k) (\widehat{\mathbf{y}}_k^{(j)} - \overline{\mathbf{y}}_k)^{\!\top}.
\end{align*}

\textit{Analysis step.}  
The Kalman gain is approximated by
\[
\widehat{\mathbf{K}}_k 
= \widehat{\mathbf{P}}_{xy}
\bigl(\widehat{\mathbf{P}}_{yy} + \mathbf{\Gamma}_k\bigr)^{-1},
\]
and the analysis ensemble is updated using the perturbed-observation formulation:
\[
\widehat{\mathbf{x}}_k^{(j)} 
= \widehat{\mathbf{x}}_{k|k-1}^{(j)} 
+ \widehat{\mathbf{K}}_k
\bigl(\mathbf{y}_k + \boldsymbol{\eta}_k^{(j)} - \widehat{\mathbf{y}}_k^{(j)}\bigr),
\qquad
\boldsymbol{\eta}_k^{(j)} \sim \mathcal{N}(\mathbf{0}, \mathbf{\Gamma}_k).
\]

The EnKF thus provides an efficient and scalable approximation to Bayesian filtering, particularly for high-dimensional systems. For linear dynamics and Gaussian noise, it converges to the classical Kalman filter as the ensemble size increases. In nonlinear settings, however, the analysis step enforces an \emph{affine} update based on sample covariances, implicitly assuming a locally linear and Gaussian posterior structure. As a result, the updated ensemble is restricted to the affine subspace spanned by the forecast ensemble, regardless of the true posterior geometry.
This structural mismatch between the Kalman update and the nonlinear state space can lead to biased estimates, loss of ensemble diversity, and filter instability or divergence~\cite{Reich2015probability,Tong2016nonlinear}. These limitations motivate a representation-based reformulation of data assimilation, in which the system dynamics and observation process are embedded into a space where linear-Gaussian assumptions are more appropriate. In the next section, we introduce a latent-space framework that achieves this goal and enables ensemble Kalman filtering to be applied in a structurally consistent manner.

\section{Latent-Space Ensemble Kalman Filter}\label{s3}

As discussed in Section~\ref{s2}, the main limitation of the ensemble Kalman filter in nonlinear settings is not the Kalman update itself, but the mismatch between its linear--Gaussian structure and the nonlinear geometry of the physical state space. Rather than modifying the Kalman update, we adopt a representation-driven perspective and seek a transformation of the state and observation variables under which the assumptions of Kalman filtering are approximately satisfied.

Specifically, we reformulate data assimilation in a learned latent space where (i) the system dynamics evolve approximately linearly under a stable operator, and (ii) observations admit a consistent linear representation. This latent formulation restores structural compatibility with ensemble Kalman filtering, allowing forecast and analysis steps to be carried out in a principled manner while retaining the computational efficiency of ensemble-based methods.

\subsection{Latent autoencoder assimilation}

Let $\mathbf{x}_k \in \mathbb{R}^D$ denote the physical state at time $t_k$. We introduce a nonlinear encoder $\mathscr{E} : \mathbb{R}^D \to \mathbb{R}^n,$
which maps the  physical state to a low-dimensional latent variable
\begin{equation}\label{encoder}
\mathbf{z}_k = \mathscr{E}(\mathbf{x}_k),
\end{equation}
together with a decoder
$
\mathscr{D} : \mathbb{R}^n \to \mathbb{R}^D,
$
which reconstructs the physical state via
\begin{equation}\label{decoder}
\widehat{\mathbf{x}}_k = \mathscr{D}(\mathbf{z}_k).
\end{equation}

The central modeling assumption is that, although the physical dynamics are nonlinear, their evolution can be approximated by a \emph{linear and stable} operator in a suitably chosen latent representation. Accordingly, we impose the latent transition model
\begin{equation}\label{latent_linear}
\mathbf{z}_k = \mathbf{A}\mathbf{z}_{k-1},
\end{equation}
\red{where $\mathbf{A} \in \mathbb{R}^{n \times n}$ is a trainable linear operator, parameterized as a weight matrix and learned jointly with the encoder and decoder through  gradient-based optimization.} This construction is closely related to Koopman-based representations \cite{Azencot2020KAE,nayak2025temporally}, in which nonlinear dynamics are lifted to linear evolution in a learned feature space. Crucially, the explicit linear structure of~\eqref{latent_linear} is not merely a modeling convenience, but a deliberate design choice that enables structural compatibility with Kalman filtering.

To avoid inconsistencies between forecast and analysis steps, the observation process is embedded into the same latent coordinate system. Let $\mathbf{y}_k \in \mathbb{R}^{D_y}$ denote the observation at time $t_k$. An observation encoder
$
\mathscr{E}_{\mathrm{obs}} : \mathbb{R}^{D_y} \to \mathbb{R}^m
$
maps observations to latent variables
\begin{equation}\label{encoder_obs}
\widetilde{\mathbf{y}}_k = \mathscr{E}_{\mathrm{obs}}(\mathbf{y}_k),
\end{equation}
which are related to the latent state through the linear model
\begin{equation}\label{latent_obs}
\widetilde{\mathbf{y}}_k = \mathbf{H}\mathbf{z}_k + \widetilde{\boldsymbol{\eta}}_k,
\qquad
\widetilde{\boldsymbol{\eta}}_k \sim \mathcal{N}(\mathbf{0}, \widetilde{\Gamma}_k).
\end{equation}
\blue{Here $\mathbf{H} \in \mathbb{R}^{m \times n}$ denotes a linear observation operator in the latent space, which can either be specified a priori or learned jointly with the observation encoder $\mathscr{E}_{\mathrm{obs}}$.} By embedding both states and observations into a common latent space, we obtain a closed linear state-space system in the latent variables. This construction eliminates the latent-space mismatch that arises when state evolution and observation mapping are learned independently, and provides a consistent foundation for Kalman-type updates.



\subsection{Network architecture and learning framework}\label{s3.2}

The latent autoencoder (LAE) framework consists of a state encoder $\mathscr{E}$, a decoder $\mathscr{D}$, a latent linear transition operator $\mathbf{A}$, an observation encoder $\mathscr{E}_{\mathrm{obs}}$, \blue{and a latent observation operator $\mathbf{H}$,} as illustrated in Fig.~\ref{fig1}. The central objective is to learn a low-dimensional representation that is expressive enough to capture nonlinear system dynamics, while enforcing a structure that is compatible with linear--Gaussian filtering. In particular, the learned latent representation is designed so that the system evolution and observation process admit linear models, thereby restoring structural consistency with ensemble Kalman filtering.

\begin{figure}[htbp] 
	\centering  
		\begin{overpic}[width=0.6\linewidth, clip=true,tics=10]{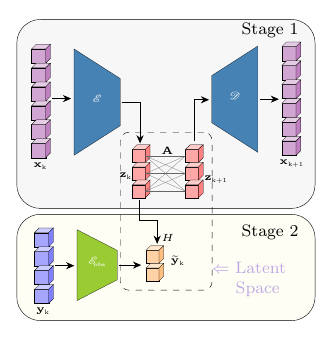}\end{overpic}
	\caption{The architecture of latent autoencoder framework. \blue{It consists of a state encoder $\mathscr{E}$ that maps physical states to a latent space, a decoder $\mathscr{D}$ that reconstructs states from latent representations, a linear transition operator $\mathbf{A}$ that governs latent dynamics, an observation encoder $\mathscr{E}_{\mathrm{obs}}$ that embeds observations into the latent space, and a latent observation operator $\mathbf{H}$, which may be specified a priori or learned.}}
	\label{fig1}
\end{figure}

To this end, we adopt a two-stage training strategy. The first stage focuses on learning a stable latent linear dynamical system from state trajectories. The second stage aligns observations with the learned latent state, ensuring consistency between forecast and analysis steps in the subsequent filtering procedure.

\paragraph{Stage I: learning stable latent linear dynamics}
Given $N_{\mathrm{traj}}$ training trajectories $\{\mathbf{x}_k^{(i)}\}$ of length $K$, the encoder $\mathscr{E}$, decoder $\mathscr{D}$, and latent transition matrix $\mathbf{A}$ are trained jointly by minimizing
\begin{align}\label{loss1}
\mathcal{L}_{\mathrm{lat}}(\mathscr{E},\mathscr{D},\mathbf{A}) = \frac{1}{N_{\text{traj}}K} \sum_{i=1}^{N_{\text{traj}}} \sum_{k=0}^{K-1} \Big(
&\lambda_{\mathrm{rec}} \|\mathscr{D}(\mathscr{E}(\mathbf{x}_k^{(i)})) - \mathbf{x}_k^{(i)}\|_2^2 \notag\\
&+ \lambda_{\mathrm{pred}} \|\mathscr{D}(\mathbf{A}\mathscr{E}(\mathbf{x}_k^{(i)})) - \mathbf{x}_{k+1}^{(i)}\|_2^2 \notag\\
&+ \lambda_{\mathrm{latent}} \|\mathbf{A}\mathscr{E}(\mathbf{x}_k^{(i)}) - \mathscr{E}(\mathbf{x}_{k+1}^{(i)})\|_2^2
\Big)+ \lambda_{\mathrm{reg}} R(\mathbf{A}),
\end{align}
where the regularization term
\[
R(\mathbf{A}) = \bigl(\max\{0, \|\mathbf{A}\|_2 - 1\}\bigr)^2,
\]
penalizes violations of a spectral norm constraint and promotes stability of the latent linear dynamics. The reconstruction term $\| \mathscr{D}(\mathscr{E}(\mathbf{x}_k^{(i)})) - \mathbf{x}_k^{(i)} \|$  enforces that the encoder--decoder pair provides a faithful representation of the physical state, ensuring that essential structures are preserved when mapping between the physical and latent spaces. The prediction term $\|\mathscr{D} (\mathbf{A}\mathscr{E}(\mathbf{x}_k^{(i)})) - \mathbf{x}_{k+1}^{(i)} \|$ encourages the decoded latent state, propagated linearly by $\mathbf{A}$, to reproduce the one-step evolution of the physical system. The latent consistency term $\| \mathbf{A}\mathscr{E}(\mathbf{x}_k^{(i)}) - \mathscr{E}(\mathbf{x}_{k+1}^{(i)})\|$ aligns the linearly propagated latent variable with the encoded latent representation of the subsequent state, thereby enforcing coherence between the learned linear dynamics and the underlying nonlinear evolution. Together, these terms ensure that the latent variables evolve approximately linearly under a stable operator, a property that is crucial for reliable sequential filtering.

\paragraph{Stage II: learning a consistent latent observation embedding}
After the latent dynamics have been identified, the encoder $\mathscr{E}$, decoder $\mathscr{D}$, and transition operator $\mathbf{A}$ are kept fixed. In the second stage, an observation encoder $\mathscr{E}_{\mathrm{obs}}$ is trained to embed observations into the same latent coordinate system. \blue{We consider a linear observation model in the latent space with operator $\mathbf{H}$, which can either be specified a priori or learned jointly with $\mathscr{E}_{\mathrm{obs}}$, and define the observation alignment loss}
\begin{equation}\label{loss2}
\mathcal{L}_{\mathrm{obs}}(\mathscr{E}_{\mathrm{obs}}) = \frac{1}{N_{\mathrm{traj}}(K+1)} \sum_{i=1}^{N_{\mathrm{traj}}} \sum_{k=0}^{K} \|\mathscr{E}_{\mathrm{obs}}(\mathbf{y}_k^{(i)}) - \mathbf{H}\mathscr{E}(\mathbf{x}_k^{(i)})\|_2^{2}\, .
\end{equation}
\blue{This objective enforces consistency between the observation encoder and the latent representation induced by $\mathscr{E}$ and $\mathbf{H}$. It can be interpreted as a regression objective toward the latent observation model, rather than a reconstruction requirement. In particular, it does not require observations to uniquely determine the underlying state, but only to produce a latent embedding compatible with the linear observation structure. When $\mathbf{H}$ is learned jointly with $\mathscr{E}_{\mathrm{obs}}$, a norm-based regularization on $\mathbf{H}$ is needed to mitigate scaling ambiguity and to ensure the stability.}

Through this two-stage training procedure, the LAE framework learns a stable latent linear dynamical system together with a consistent embedding of states and observations. The result is a closed linear state-space model in the latent variables, in which the assumptions underlying Kalman filtering are approximately satisfied. In the next subsection, we integrate the ensemble Kalman filter directly into this learned latent representation, leading to the proposed LAE-EnKF.

\subsection{Latent autoencoder ensemble Kalman filter}\label{s3-3}

Once training is completed, the learned latent dynamics and observation embeddings define a linear state-space model in the latent variables. The ensemble Kalman filter can therefore be applied directly in the latent space, where its structural assumptions are approximately satisfied.

In the proposed LAE-EnKF, the forecast step propagates latent ensemble members using the learned linear operator $\mathbf{A}$, while the analysis step performs Kalman-type updates using latent observations generated by $\mathscr{E}_{\mathrm{obs}}$. The updated latent ensemble is then decoded to recover full-order state estimates in the physical space. A detailed algorithmic description is provided in Algorithm~\ref{alg1}.

\begin{algorithm}[htbp]
	\caption{Latent Autoencoder Ensemble Kalman Filter (LAE-EnKF)}
	\label{alg1}
	\begin{algorithmic}[1]
	\Require {
		Training data $\{(\mathbf{x}_k,\mathbf{y}_k)\}$; 
		initial physical ensemble $\{\mathbf{x}_0^{(j)}\}_{j=1}^{N_e}$.
	}
	\Ensure {
		State estimates $\{\widehat{\mathbf{x}}_k\}$, trained $\mathscr{E}$, $\mathscr{D}$, 
		transition matrix $\mathbf{A}$, and observation encoder $\mathscr{E}_{\mathrm{obs}}$.
	}
	\State {\textbf{Offline training (LAE model)}}
	\State Train $(\mathscr{E},\mathscr{D},\mathbf{A})$ by minimizing the empirical latent loss ${\mathcal{L}}_{\mathrm{lat}}$ in~\eqref{loss1}.
	\State Train $\mathscr{E}_{\mathrm{obs}}$ by minimizing the observation-alignment loss ${\mathcal{L}}_{\mathrm{obs}}$ in~\eqref{loss2}.
	
	\State {\textbf{Online assimilation (LAE-EnKF)}}
	\State \textbf{Initialization:}  $\widehat{\mathbf{z}}_0^{(j)} = \mathscr{E}(\mathbf{x}_0^{(j)}), 
	\quad j = 1,\dots,N_e.$
	
	\For{$k = 1,2,\dots,K$}
	\State 	\textbf{Forecast step:} $\widehat{\mathbf{z}}_{k|k-1}^{(j)} 
		= \mathbf{A}\,\widehat{\mathbf{z}}_{k-1}^{(j)},\quad j = 1,\cdots,N_e$.
		
		\State Compute latent observation: 
		$\widehat{\mathbf{y}}_k^{(j)} = \mathbf{H}\purp{\widehat{\mathbf{z}}}_{k|k-1}^{(j)}, \quad j=1,\cdots,N_e$. 	
		\State Compute latent forecast mean:
		$\overline{\mathbf{z}}_{k|k-1} = \frac{1}{N_e}\sum\limits_{j=1}^{N_e}\widehat{\mathbf{z}}_{k|k-1}^{(j)}$. 
	\State	Compute latent observation mean:
		$\overline{\mathbf{y}}_{k} = \frac{1}{N_e}\sum\limits_{j=1}^{N_e}\widehat{\mathbf{y}}_k^{(j)} $. 
	\State	Compute empirical latent covariances:
		\[
		\mathbf{P}_{zy} = \frac{1}{N_e-1}\!\sum\limits_{j=1}^{N_e}
		(\widehat{\mathbf{z}}_{k|k-1}^{(j)} - \overline{\mathbf{z}}_{k|k-1})
		(\widehat{\mathbf{y}}_k^{(j)} - \overline{\mathbf{y}}_k)^\top,
		\]
		\[
		\mathbf{P}_{yy} = \frac{1}{N_e-1}\!\sum\limits_{j=1}^{N_e}(\widehat{\mathbf{y}}_k^{(j)} - \overline{\mathbf{y}}_k)(\widehat{\mathbf{y}}_k^{(j)} -\overline{\mathbf{y}}_k)^\top.
		\]
		\State Encode observation: 
		$\widetilde{\mathbf{y}}_k = \mathscr{E}_{\mathrm{obs}}(\mathbf{y}_k).$ 
	\State	\textbf{Analysis step:}
		$$\widehat{\mathbf{z}}_k^{(j)} = \widehat{\mathbf{z}}_{k|k-1}^{(j)} + \mathbf{P}_{zy}\big(\mathbf{P}_{yy} + \widetilde{\Gamma}\big)^{-1} \big( \widetilde{\mathbf{y}}_k + \widetilde{\boldsymbol{\eta}}_k^{(j)} - \widehat{\mathbf{y}}_k^{(j)}  \big),\quad j=1,\cdots,N_e.$$ 
	\EndFor
	\State \textbf{Reconstruction:} $
	\widehat{\mathbf{x}}_k = \mathscr{D}\!\left( \frac{1} {N_e}\sum\limits_{j=1}^{N_e}\widehat{\mathbf{z}}_k^{(j)}\right).
$
	\end{algorithmic}
\end{algorithm}

By shifting data assimilation from the nonlinear physical space to a compact latent space governed by stable linear dynamics, the LAE-EnKF preserves the scalability and efficiency of the classical EnKF while avoiding the structural mismatch that limits its performance in nonlinear settings. This representation-driven approach provides a principled pathway for extending ensemble Kalman filtering to high-dimensional nonlinear systems.

\section{Theoretical Analysis}\label{s4}
In this section, we provide a performance guarantee of the proposed LAE. Our LAE is motivated by the manifold hypothesis \cite{Tenenbaum2000science, khoo2024temporal}, which suggests that although the full-order states $\mathbf{x}\in \mathbb{R}^{D}$ is high-dimensional, the set of dynamically attainable states concentrates on a low-dimensional geometric structure. In particular, we assume that the data lie on an $n$-dimensional compact smooth Riemannian manifold $\mathcal{M}$ that is isometrically embedded in the ambient space $\mathbb{R}^{D}$ \cite{lee2006riemannian}. The compactness of $\mathcal{M}$ implies uniform boundedness, there exists $B>0$ such that
\[
\|\mathbf{x}\|_\infty \le B ,\qquad\forall \mathbf{x} \in \mathcal M.
\]
Furthermore, we assume that the manifold $\mathcal{M}$ admits a global smooth parametrization, so that it can be represented by a single coordinate system. This corresponds to the single-chart setting in the chart autoencoder framework of \cite{LIU2024DAE}. Within this framework, the encoder--decoder pair $(\mathscr{E},\mathscr{D})$ can be viewed as approximating a global chart and its inverse on $\mathcal{M}$, while the latent operator $\mathbf{A}$ acts on a Euclidean representation of the underlying manifold.

To derive an upper bound on the generalization error for the LAE, we impose the following assumptions. 
\begin{assumption}\label{assum1}
	Let $\mathcal{M} \subset \mathbb{R}^{D}$ be a compact $n$-dimensional $C^{2}$ submanifold. Assume that $\mathcal M$ admits a global bi-Lipschitz parameterization $\varphi:\mathcal M \rightarrow [-\Lambda,\Lambda]^{n}$ \red{for some $\Lambda>0$,} with inverse $\phi=\varphi^{-1}:\varphi(\mathcal M)\rightarrow\mathcal M$. That is, there exist constants $C_{\varphi},C_{\phi}>0$ such that for all $\mathbf{v}_1,\mathbf{v}_2\in\mathcal M$,
	\[
	C_{\varphi}^{-1}\|\mathbf{v}_1-\mathbf{v}_2\|_2 \le \|\varphi(\mathbf{v}_1)-\varphi(\mathbf{v}_2)\|_2 \le C_{\varphi}\|\mathbf{v}_1-\mathbf{v}_2\|_2,
	\]
	and
	\[
	\|\phi(\mathbf{z}_1)-\phi(\mathbf{z}_2)\|_2 \le C_{\phi}\|\mathbf{z}_1-\mathbf{z}_2\|_2, \qquad \mathbf{z}_1,\mathbf{z}_2\in\varphi(\mathcal M).
	\]
	Furthermore, assume that the one-step dynamics $\mathbf F:\mathcal M\to\mathbb R^D$ admits a latent linear representation, i.e., there exists a matrix $\mathbf A^\star\in\mathbb R^{n\times n}$ with $\|\mathbf A^\star\|_2\le\rho$ for a fixed constant $\rho\le 1$, such that
	\[
	\mathbf F(\mathbf x)=\phi\!\left(\mathbf A^\star\varphi(\mathbf x)\right),
	\qquad \forall\,\mathbf x\in\mathcal M .
	\]
\end{assumption}

We assume that all physical states used for training satisfy $\mathbf{x}_k \in \mathcal{M}$. Although the trajectory $\{\mathbf{x}_k\}$ is generated by an underlying time-dependent Markov process, we further assume that the sampling interval exceeds the mixing time of the dynamics so that statistical dependence across samples becomes negligible, while we remark that our learning result may hold even for non i.i.d. data source (see \cite{che2024stochastic} and the references within). Under this standard assumption in learning theory for dynamical systems, the data may be treated as approximately independent for the purpose of deriving finite-sample generalization bounds.

\begin{assumption}\label{iid_sampling}
Let $\{(\mathbf{x},\mathbf{x}^{+})\}_{k=1}^N$ denote the input-target pairs, where \red{$\x_k^{+} = \mathbf{F}(\x_k) + \bfw_k$ and $\bfw_k$ denotes an additive perturbation term, assumed to be i.i.d., independent of $\mathbf{x}_k$, with $\mathbb{E}[\mathbf{w}_k]=0$ and $\mathbb{E}\|\mathbf{w}_k\|_2^2 < \infty$. We assume that $(\mathbf{x},\mathbf{x}^+)$ are sampled i.i.d. from a joint distribution $\mathbb{P}_{X,X^+}$ supported on a neighborhood of $\mathcal{M}\times\mathcal{M}$.} 
\end{assumption}

For the loss function \eqref{loss1}, the weights $\lambda_{\mathrm{rec}}, \lambda_{\mathrm{pred}}, \lambda_{\mathrm{latent}}, \lambda_{\mathrm{reg}}$ are tuning parameters used in practice to balance the relative scales of reconstruction, one-step prediction, latent consistency, and stability regularization. \red{From a theoretical perspective,  they act as positive multiplicative constants and can be absorbed into the leading constants of the generalization bound, and hence do not affect the convergence rate. Without loss of generality, we set $\lambda_{\mathrm{rec}}=\lambda_{\mathrm{pred}}=\lambda_{\mathrm{latent}}=\lambda_{\mathrm{reg}}=1$ and consider the simplified population loss:}
\begin{align}\label{loss_pop}
	\mathcal{L}_{\mathrm{lat}}^{\mathrm{pop}}(\mathscr{E},\mathscr{D},\mathbf{A}) :=\mathbb{E}_{(\mathbf{x},\mathbf{x}^+)\sim\mathbb{P}_{X,X^+}} \Big[ & \| \mathscr{D} (\mathbf{A}\mathscr{E} (\mathbf{x})) - \mathbf{F}(\mathbf{x}) \red{-\bfw} \|_2^2 + \|\mathscr{D}(\mathscr{E}(\mathbf{x})) -\mathbf{x}\|_2^2 \notag\\
	&+ \|\mathbf{A}\mathscr{E}(\mathbf{x})-\mathscr{E}(\mathbf{F}(\mathbf{x})\red{+\bfw})\|_2^2+R(\mathbf{A}) \Big].
\end{align}

\red{Let $\mathcal{S} = \{(\x_i, \x_i^+, \y_i)\}_{i=1}^N$ denote i.i.d. samples drawn from a joint distribution $\mathbb{P}_{X,X^+,Y}$ induced by the underlying dynamics and the observation model.} Under these assumptions, we obtain the following finite-sample upper bound for the squared generalization error of the LAE.
\begin{theorem}\label{thm1}
	Suppose that Assumptions~\ref{assum1} and~\ref{iid_sampling} hold. Let $(\widehat{\mathscr{E}},\,\widehat{\mathbf{A}},\,\widehat{\mathscr{D}})$ denote a global minimizer of the empirical objective~\eqref{loss1} with all $\lambda\equiv 1$. Assume that the encoder--decoder class has sufficient approximation capacity to realize a global smooth parametrization of the manifold $\mathcal{M}$ and its inverse. Then \red{for any fixed $C_1>1$}, the squared generalization error satisfies
	\begin{equation}\label{loss1_sge}
		\mathbb{E}_{\mathcal{S}} \!\left[ \mathcal{L}_{\mathrm{lat}}^{\mathrm{pop}} (\widehat{\mathscr{E}},\widehat{\mathscr{D}},\widehat{\mathbf{A}}) \right] \le \red{C_0} D^{2}\log^{2} D N^{-\frac{2}{n+2}}\log^{4} N + \red{C_1(1+C_1^2C^2_{\varphi})\E\|\bfw\|_2^2},
	\end{equation}
	for some constant $\red{C_0}>0$ depending only on $n$, $\red{C_1}$, $B$, $\Lambda$, $\rho$, the Lipschitz constants of the parametrization maps, and the volume of $\mathcal{M}$.
\end{theorem}

 A detailed proof is provided in Appendix \ref{appendixA}. \red{The upper bound contains the term $\mathbb{E}\|\mathbf{w}\|_2^2$, which represents the irreducible error induced by stochastic perturbations in the dynamics. For deterministic systems, where the evolution is noise-free, this term vanishes.}  Our analysis follows the general proof strategy developed for the single-chart autoencoder framework \cite{LIU2024DAE}, where the three components of the population loss \eqref{loss_pop} are bounded separately. Compared with \cite{LIU2024DAE}, our setting includes an additional latent linear evolution term $\mathbf{A}$, which introduces an extra intermediate factor of order $\mathcal{O}(N^{-\frac{2}{n+2}})$ in the estimation error. However, this additional term does not affect the overall convergence rate.  \red{Consequently, the squared generalization error consists of an estimation error term of order $\mathcal{O}\bigl(N^{-\frac{2}{n+2}}\log^4 N\bigr)$ together with an irreducible error term due to the stochastic perturbations in the dynamics.}

The following theorem establishes a generalization bound for the observation encoder $\mathscr E_{obs}$ defined above, \red{which aims to approximate the latent representation $\mathbf H\,\mathscr E(\mathbf{x})$ from observations $\mathbf{y}$. To ensure identifiability and stability of this mapping, we assume that the observation operator $\mathcal H$ satisfies a local stability condition.}

\begin{theorem}\label{thm2}
Suppose that Assumption~\ref{assum1} holds, and that the pre-trained components $\mathscr{E}$, $\mathscr{D}$, and $\mathbf{A}$ are fixed. Let $\widehat{\mathscr{E}}_{\mathrm{obs}}$ denote a global minimizer of the observation alignment objective~\eqref{loss2}. Assume that the observation encoder class has sufficient approximation capacity to represent \red{the regression function $\mathbf{y} \;\mapsto\; \mathbb{E}[\mathbf{H} \mathscr E(\mathbf{x}) \mid \mathbf{y}]$, under the data distribution induced by the observation process. Then the expected generalization error satisfies }
\begin{align}
\mathbb{E}_{\mathcal{S}} \mathbb{E}_{(\mathbf{x},\mathbf{y})} \left[ \bigl\| \widehat{\mathscr{E}}_{\mathrm{obs}}(\mathbf{y}) - \red{\mH}\mathscr{E}(\x) \bigr\|_2^{2}\right]  \le  \red{C_0^{obs}} D_y^{2}\log^{2} D_y N^{-\frac{2}{n+2}}\log^{4} N \red{+ C_1^{obs} \mathbb{E}\|\boldsymbol{\eta}\|_2^2},
\end{align}
where the constant \red{$C_0^{obs}>0$ depends on $n$, $R_{obs}$, $\Lambda$, and the volume of $\mathcal{M}$, and $C_1^{obs}>0$ depends on $\|\mH\|_2$, $C_\phi$ and the stability constant $C_{\mathcal H}$.}
\end{theorem}

\red{A detailed proof is provided in Appendix~\ref{appenB}. The error bound consists of an estimation term of order $\mathcal{O}\bigl(N^{-\frac{2}{n+2}}\log^4 N\bigr)$ and an irreducible term induced by the observation noise $\boldsymbol{\eta}$. The rate is governed by the intrinsic dimension $n$ and matches that in Theorem~\ref{thm1}.}

\section{Numerical Results}\label{s5}
In this section, we present numerical experiments to assess the performance of the proposed LAE-EnKF in nonlinear data assimilation problems. The experiments are designed to evaluate three key aspects of the method: assimilation accuracy, robustness over time, and the benefit of enforcing linear dynamics in a learned latent space. All of the following numerical examples are carried out on a computing server with a NVIDIA GeForce RTX 3090 GPU.
 To enrich the information content of partial observations, we adopt a time-delay embedding strategy~\cite{Takens1981timedelay}. Specifically, a sequence of $L$ consecutive observations is concatenated to form an augmented observation vector,
\[
\mathbf{y}_k^{(L)} = \big[ \mathbf{y}_{k-L+1}^\top, \ldots, \mathbf{y}_{k-1}^\top, \mathbf{y}_k^\top \big]^\top,
\]
which is then processed by the observation encoder $\mathscr{E}_{\mathrm{obs}}$. \blue{This construction allows the latent observation embedding to exploit temporal correlations and improves identifiability by incorporating short-term dynamics.}

To quantify assimilation performance, we report several standard error metrics. The relative state error at time step $k$ is defined as
\begin{equation}\label{nrmse}
e_{\mathrm{Rel},k}
= \frac{\|\widehat{\mathbf{x}}_k - \mathbf{x}_k\|_2}{\|\mathbf{x}_k\|_2},
\end{equation}
where $\widehat{\mathbf{x}}_k$ denotes the estimated state and $\mathbf{x}_k$ is the reference solution. To measure overall accuracy over a finite time window, we further define the time-averaged root-mean-square error and the corresponding global relative error,
\begin{equation}\label{nrmse_total}
\red{\mathrm{E}_{1:K}
= \sqrt{\frac{1}{KD}\sum_{k=1}^{K}\sum_{i=1}^{D}
\big(\widehat{x}_{k,i}-x_{k,i}\big)^2},
\quad
\mathrm{E}_{\mathrm{Rel},1:K}
= \frac{\mathrm{E}_{1:K}}{
\sqrt{\frac{1}{K}\sum_{k=1}^{K}
\left(\frac{1}{D}\sum_{i=1}^{D}x_{k,i}^2\right)} }.}
\end{equation}
Here, $x_{k,i}$ and $\widehat{x}_{k,i}$ denote the $i$th components of the true and estimated state vectors at time step $k$, respectively.



To highlight the role of the latent linear structure and latent-space filtering, we compare the proposed LAE-EnKF with the following baseline methods.
\red{
\begin{itemize}
\item[(i)] \textbf{EnKF}: the standard ensemble Kalman filter, where both forecast and analysis steps are carried out directly in the physical state space (see Section~\ref{s2.2}).
\item[(ii)] \textbf{AE-EnKF}: an autoencoder-assisted EnKF where the encoder-decoder pair  is used to reconstruct the state during the forecast step, while the EnKF update is performed in the physical space (see Algorithm~\ref{alg:AE-EnKF}).
\item[(iii)] \textbf{DAE-EnKF}: a latent-space EnKF method in which the latent dynamics are modeled by a general nonlinear mapping, and both forecast and analysis steps are carried out in the latent space (see Algorithm~\ref{alg:DAE-EnKF}). This setting is similar in spirit to~\cite{LI2024LAINR}. To ensure a fair comparison, the encoder and decoder architectures are identical to those used in the proposed LAE framework.
\end{itemize}}

All methods use the same ensemble size, data normalization, and training configuration to ensure comparability. In particular, all state variables and observations are normalized to zero mean and unit variance prior to training and assimilation. \blue{ Unless otherwise specified, we use the identity latent observation operator $\mathbf{H}=\mathbf{I}$ throughout}. In all experiments, the encoder and decoder are implemented using convolutional neural networks (CNN), which are well suited for high-dimensional systems with spatial structure. The state encoder consists of three convolutional layers. Each layer uses a $4\times4$ kernel with stride $2$, progressively reducing the spatial resolution while extracting multiscale features. A fully connected layer with \emph{LeakyReLU} activations (negative slope  $\alpha = 0.1$) then maps the resulting features to the latent variables. The decoder mirrors this architecture and reconstructs the physical state using transposed convolutional layers. The latent linear transition operator $\mathbf{A}$ is learned as a single fully connected neural network without any activation function. This architectural choice explicitly enforces linear latent dynamics and is essential for maintaining structural compatibility with Kalman filtering. The observation encoder is implemented as a lightweight network composed of two one-dimensional convolutional layers with kernel size $3$, followed by a fully connected layer that projects the observations into the latent space. This design allows observations with different spatial layouts to be embedded consistently with the latent state representation. The loss weights in~\eqref{loss1} are selected so that the reconstruction, prediction, latent consistency, and regularization terms contribute at comparable magnitudes during training. All networks are implemented in PyTorch and trained using the Adam optimizer with an initial learning rate of $10^{-3}$. Early stopping based on the validation RMSE is applied to mitigate overfitting.

\subsection{Toy example}
We first consider a 100-dimensional nonlinear dynamical system that provides a controlled setting for examining the role of latent linear structure in data assimilation. Although the observed state is high dimensional, the underlying dynamics are intrinsically low dimensional, making this example well suited for assessing representation quality, latent dynamics, and filtering performance.

The system is governed by a single angular variable $\theta_{\red{k}}$, while the physical state is defined as
\begin{equation}
\mathbf{x}_{\red{k}}
= W
\begin{bmatrix}
\cos(\theta_{\red{k}}) \\
\sin(\theta_{\red{k}})
\end{bmatrix}
\in \mathbb{R}^{100},
\qquad
W \in \mathbb{R}^{100\times 2},
\quad
W_{ij} \sim \mathcal{N}(0,1),
\end{equation}
which embeds a two-dimensional rotational motion into a 100-dimensional ambient space. \red{Here, $k=0,1,\dots,K$ denotes discrete time steps.} The angular dynamics evolve according to
\begin{equation}
\theta_{\red{k}+1}
= \theta_{\red{k}} + \delta + \alpha \sin(2\theta_{\red{k}}) + c\,\epsilon_{\red{k}},
\end{equation}
where $\delta=\pi/50$ is a constant rotation increment, $\alpha=0.01$ introduces weak nonlinearity, $c=0.01$ scales the process noise, and $\epsilon_{\red{k}}\sim\mathcal{N}(0,1)$.

Observations are generated through a linear measurement model,
\begin{equation}
\mathbf{y}_{\red{k}}
= \purp{\mathcal{H}}(\mathbf{x}_{\red{k}}) + \boldsymbol{\eta}_{\red{k}},
\qquad
\boldsymbol{\eta}_{\red{k}} \sim \mathcal{N}(0,0.1^2),
\end{equation}
where \purp{$\mathcal{H}:\mathbb{R}^{100}\to\mathbb{R}^2$ is a linear observation operator that} randomly selects two state components. We simulate $N_{\mathrm{traj}}=500$ independent trajectories of length ${\red{K}}=100$, with initial angles $\theta_0$ sampled uniformly from $[-\pi,\pi]$. The data are split into $90\%$ for training and $10\%$ for testing. 

To isolate the effect of enforcing linear latent dynamics, we first examine the geometric structure of the latent representations learned by the proposed LAE and by the DAE-EnKF model, referred to here as DAE for brevity. Fig.~\ref{41.1} shows the latent variables for latent dimensions $n=2,3,4$. The DAE produces distorted closed curves, indicating that the learned latent dynamics do not preserve the underlying rotational structure. In contrast, the LAE generates significantly smoother latent trajectories, indicating that the imposed linear evolution promotes more regular and coherent temporal behavior in the latent space. When the latent dimension exceeds the intrinsic dimension ($n=3,4$), we project the latent variables onto their first two principal components for visualization. In all cases, the LAE yields nearly isotropic circles in the projected space, showing that redundant latent coordinates are effectively suppressed. Fig.~\ref{41.2} further displays pairwise latent coordinates $(z_i,z_j)$ for $n=3$. Each pair forms an ellipse, confirming that all latent variables evolve coherently within a shared two-dimensional linear subspace. Together, these results indicate that the LAE preserves the intrinsic manifold structure while enforcing global linear consistency in the latent dynamics.

\begin{figure}[htbp] 
	\centering  
		\begin{overpic}[width=0.7\linewidth, clip=true,tics=10]{ 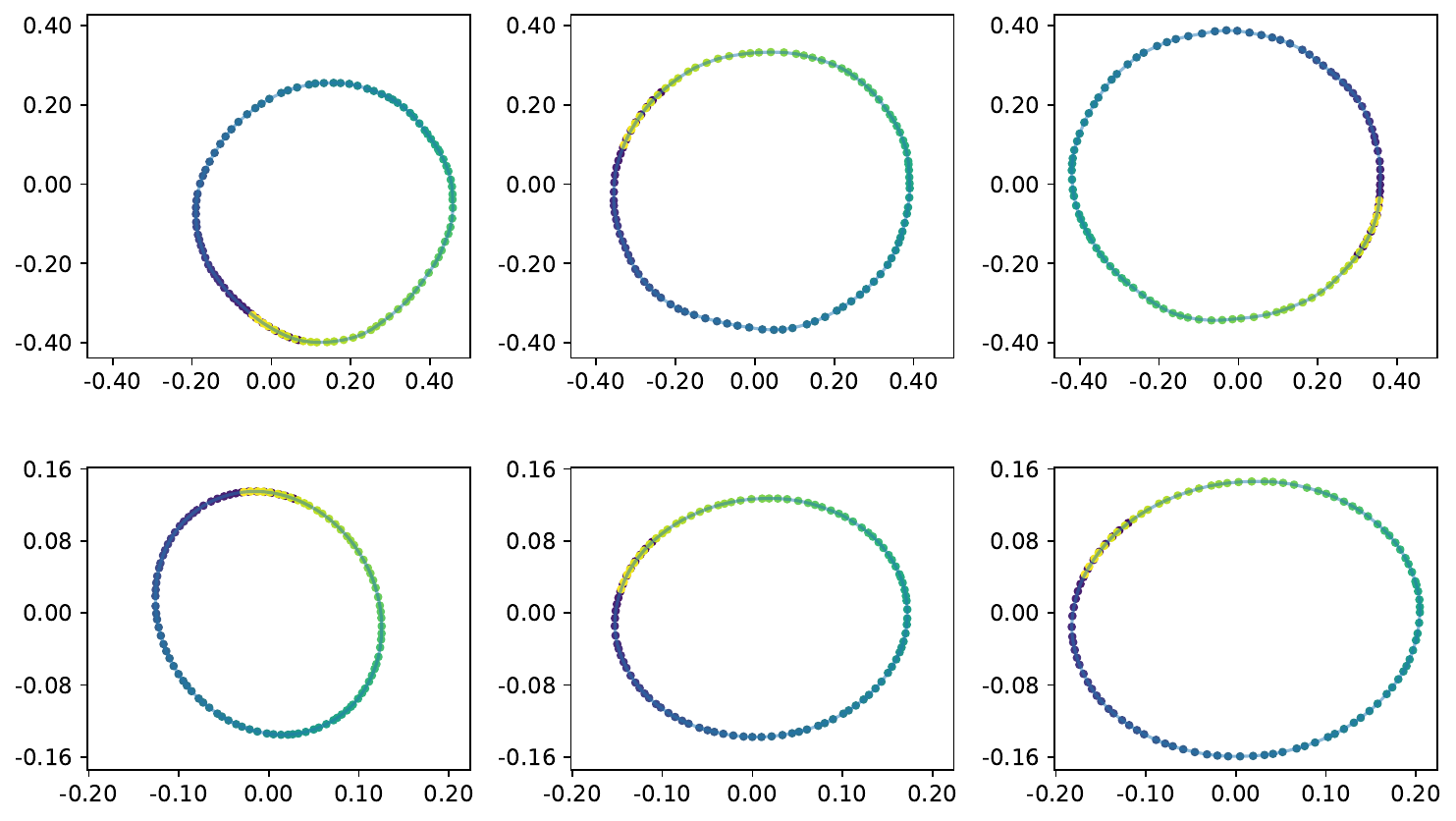}\put (15,56.5) {\footnotesize $n=2$ }\put (47,56.5) {\footnotesize $n=3$ }\put (81,56.5) {\footnotesize $n=4$ }\end{overpic}
	\caption{Learned latent representations for latent dimensions $n=2,3,4$. Top: latent variables learned by the autoencoder component of DAE-EnKF (denoted as DAE), where no structural constraint is imposed on the latent dynamics. Bottom: latent variables learned by the proposed LAE, which enforces linear evolution in the latent space. For $n>2$, the latent variables are projected onto their first two principal components using PCA for visualization.}  
	\label{41.1}
\end{figure}

\begin{figure}[htbp] 
	\centering  
		\begin{overpic}[width=0.9\linewidth, clip=true,tics=10]{ 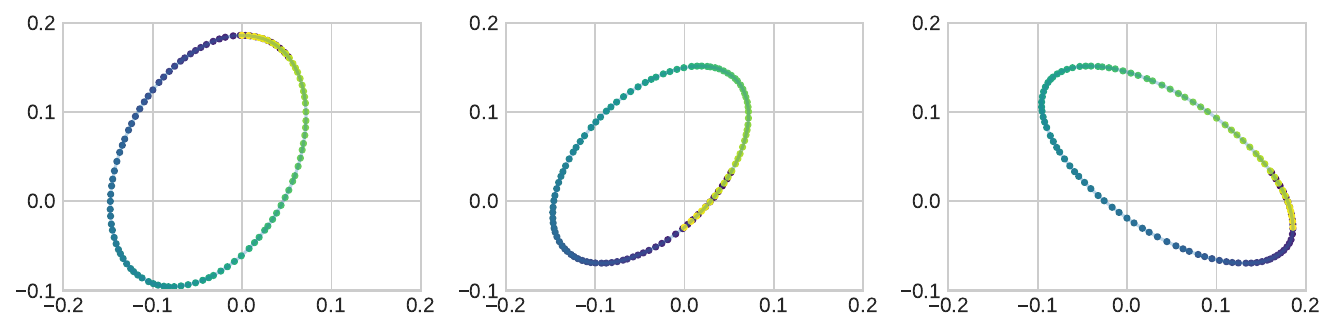}\put (14.5,24.8) {\footnotesize $(z_1,z_2)$ }\put (47,24.8) {\footnotesize $(z_1,z_3)$ }\put (80.5,24.8) {\footnotesize $(z_2,z_3)$ }\end{overpic}
	\caption{ Pairwise latent coordinate plots $(z_i, z_j)$ for $n=3$ learned by the LAE.}
	\label{41.2}
\end{figure}

We next study the impact of latent dimension on long-term prediction. Fig.~\ref{41.3} reports the relative prediction error obtained by propagating the learned latent dynamics forward in time and decoding back to the physical space.  For latent dimensions $n=2,3,4$, the prediction errors follow very similar trajectories over the entire forecast horizon, indicating that the long-term accuracy is only weakly affected by the choice of $n$. Although the underlying system is intrinsically two-dimensional, slightly higher-dimensional latent spaces do not degrade the prediction performance.  These observations highlight an important feature of the proposed method: once the latent dimension is sufficient to represent the essential dynamics, the enforced linear structure prevents spurious modes from destabilizing long-term predictions. As a result, the method is both dimension-efficient and robust, achieving accurate predictions without requiring careful tuning of the latent dimension.

\begin{figure}[tbp] 
	\centering  
		\begin{overpic}[width=0.45\linewidth, clip=true,tics=10]{ 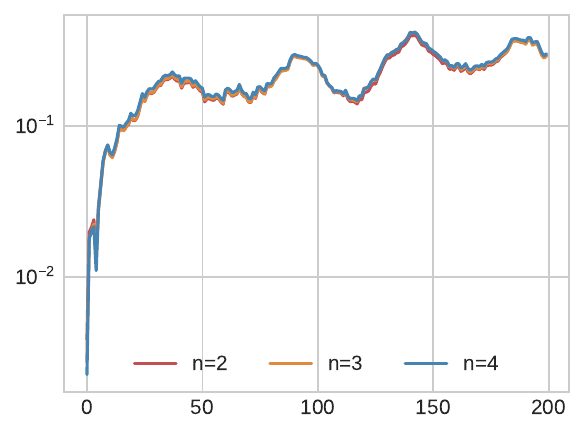}\end{overpic}  
    \caption{Long-term prediction relative error for latent dimensions $n=2,3,4$. Predictions are obtained by recursively propagating the learned latent dynamics over time steps and decoding back to the physical space.}
	\label{41.3}
\end{figure}

Finally, we evaluate data assimilation performance using LAE-EnKF and compare it with EnKF, AE-EnKF, and DAE-EnKF. \red{All methods exhibit comparable computational costs within the same order of magnitude, run-time comparisons are therefore not reported.} Tab. \ref{41.4} reports the global relative RMSE $\mathrm{E}_{\mathrm{Rel},1:\red{K}}$ for latent dimensions $n=2,3,4$. \red{Fig.~\ref{41.5} shows the evolution of the reconstruction error over discrete time steps} for $n=2$, together with $95\%$ confidence intervals over ten independent runs. \red{All methods outperform the standard EnKF. AE-EnKF attains low error in the initial stage, followed by oscillations and eventual error growth. This is due to its reliance on a one-step learned mapping, which lacks temporal consistency under repeated application, leading to error accumulation. In contrast, DAE-EnKF shows noticeable fluctuations, while LAE-EnKF produces a smoother and more stable error trajectory.}
\begin{table}[tbp]
	\caption{Relative error {$\mathrm{E}_{\text{Rel},1:\red{K}}$} for different methods with various values of  $n$.}
	\label{41.4}
	\centering
	\begin{tabular}{
			>{\centering\arraybackslash}m{1.3cm}
			>{\centering\arraybackslash}m{1.9cm}
			>{\centering\arraybackslash}m{1.9cm}
			>{\centering\arraybackslash}m{2.1cm}
			>{\centering\arraybackslash}m{2.1cm}
		}
		\toprule
		\toprule
		$\boldsymbol{n}$ & \textbf{EnKF} & \textbf{AE-EnKF} & \textbf{DAE-EnKF} & \textbf{LAE-EnKF} \\ \midrule
		-- &1.9757E-01 &--&--&--\\
		$n=2$ &-- & 1.6412E-01 & 2.3312E-01 & \textbf{3.8947E-02} \\
		$n=3$ & -- & 1.8667E-01 & 2.1106E-01 & \textbf{1.4244E-02} \\
		$n=4$ & -- & 1.9632E-01 & 4.0424E-01 & \textbf{1.5682E-02} \\
		\bottomrule
	\end{tabular}
\end{table}

\begin{figure}[htbp] 
	\centering  
		\begin{overpic}[width=0.45\linewidth, clip=true,tics=10]{ 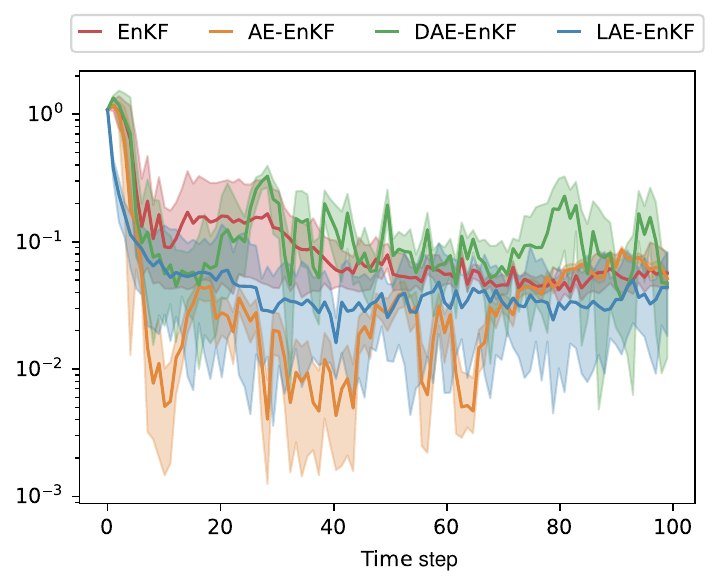}\end{overpic}
    \caption{\red{Evolution of the relative RMSE over discrete time steps for different methods with latent dimension $n=2$.}  Solid curves show the mean over 10 independent runs, and shaded regions indicate the corresponding 95\% confidence intervals.}
	\label{41.5}
\end{figure}

Fig.~\ref{41.6} compares reconstructed state trajectories in representative coordinate pairs $(x_i,x_j)$. The standard EnKF and DAE-EnKF deviate significantly from the true orbit. AE-EnKF better captures the overall geometry but still show intermittent departures. The LAE-EnKF reconstruction remains closely aligned with the true trajectory, yielding smoother and more coherent state evolution. These results demonstrate that enforcing linear latent dynamics improves both predictive stability and assimilation robustness, directly supporting the design motivation of the proposed LAE-EnKF.

\begin{figure}[htbp] 
\vspace{0.2cm}
	\centering  
		\begin{overpic}[width=0.96\linewidth, clip=true,tics=10]{ 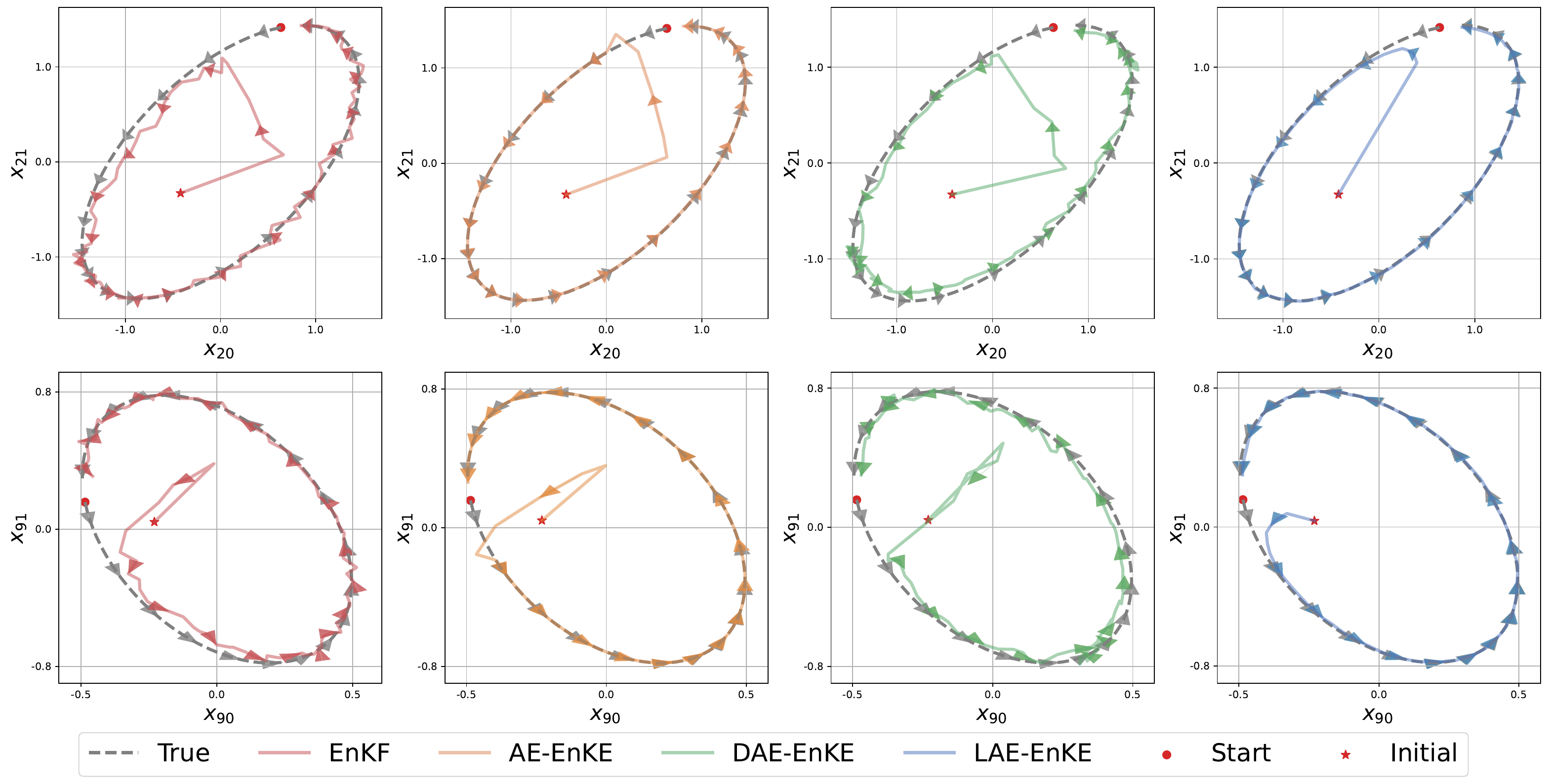}\put (11,51) {\footnotesize EnKF }\put (35,51) {\footnotesize AE-EnKF }\put (59,51) {\footnotesize DAE-EnKF }\put (84,51) {\footnotesize LAE-EnKF }\end{overpic}
	\caption{Reconstructed state trajectories in two representative coordinate pairs $(x_i, x_j)$ from the $D=100$ dimensional physical state for different methods with $n=2$, compared with the true trajectory. Red stars indicate the initial reconstructed state, and red dots denote the initial true state.}\label{41.6}
\end{figure}
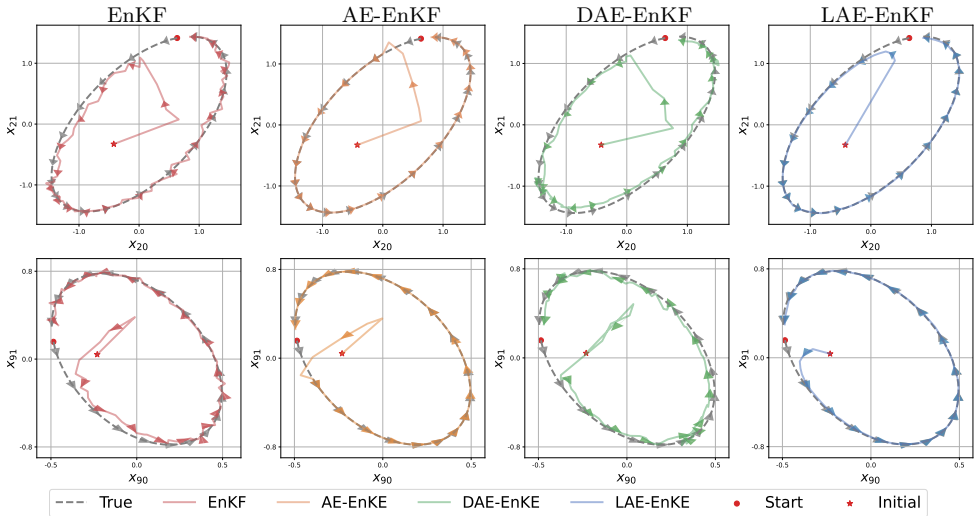

\subsection{ Advection--diffusion--reaction equation}\label{ex2}
We consider a nonlinear advection diffusion reaction equation on the two-dimensional spatial domain $\purp{\boldsymbol{x}} =(x_1,x_2) \in [0,1]\times[0,1]$,
\begin{align}
\partial_t u(\purp{\boldsymbol{x}},t) + \mathbf{v}(\purp{\boldsymbol{x}})\cdot\nabla u(\purp{\boldsymbol{x}},t) &= \mu \Delta u(\purp{\boldsymbol{x}},t) - \alpha\bigl(\sin u(\purp{\boldsymbol{x}},t)-u(\purp{\boldsymbol{x}},t)\bigr),\\
	u(\purp{\boldsymbol{x}},0) &= u_0(\purp{\boldsymbol{x}}), \notag
\end{align}
subject to periodic boundary conditions. The diffusion coefficient and reaction parameter are set to \blue{$\mu=10^{-5}$} and $\alpha=0.8$, respectively. The advection velocity field is given by
\[
\mathbf{v}(\purp{\boldsymbol{x}}) = \bigl(\sin(2\pi x_2),\, \sin(2\pi x_1)\bigr).
\]

The initial condition is modeled as a zero-mean periodic Gaussian random field,
\begin{equation}
	u(\purp{\boldsymbol{x}},0) \sim \mathcal{GP}\bigl(0,K_\ell(\purp{\boldsymbol{x}},\purp{\boldsymbol{x}}')\bigr),
\end{equation}
with covariance kernel
\[
K_\ell(\purp{\boldsymbol{x}},\purp{\boldsymbol{x}}') =  \exp\!\left( - \frac{\sin^2(\pi(x_1-x_1'))+\sin^2(\pi(x_2-x_2'))}{2\ell^2} \right).
\]

Data assimilation is performed using sparse and noisy observations of the true solution. Measurements are collected at $5\times5$ uniformly distributed sensor locations, resulting in an observation vector $\mathbf{y}_k\in\mathbb{R}^{25}$, \blue{corrupted by additive Gaussian noise with standard deviation $\sigma=0.01$. Observations are available every $\Delta t=0.05$ up to the final time $T=3$. Fig.~\ref{44.1} shows the true solution at the initial and final times, with sensor locations marked by white circles.} The initial ensemble is generated by sampling spatial white noise followed by convolution with a two-dimensional Gaussian filter with standard deviation $3$. This procedure introduces spatial correlation and provides a common prior ensemble for all filtering methods.

\begin{figure}[htbp] 
	\centering  
		\begin{overpic}[width=0.44\linewidth, clip=true,tics=10]{ 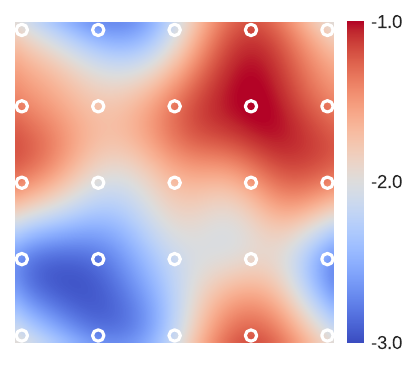}\end{overpic}
		\begin{overpic}[width=0.44\linewidth, clip=true,tics=10]{ 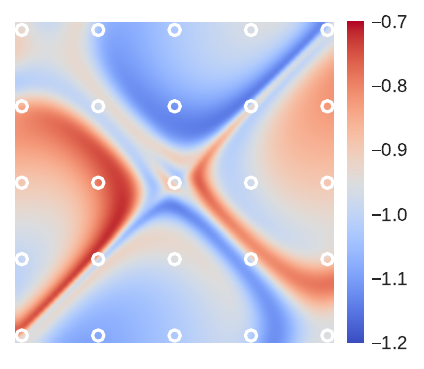}\end{overpic}
	\caption{\blue{True solution fields of the advection--diffusion--reaction equation at the initial time (Left) and final time $T=3$ (Right). White circles indicate the observation locations.}}
	\label{44.1}
\end{figure}

\blue{To better understand the role of different components in the proposed framework, we first conduct an ablation study on the loss function~\eqref{loss1}. Here, we fix the latent dimension to $n=20$ and remove either the latent consistency term or the prediction term in~\eqref{loss1}. Fig.~\ref{44.10} shows the multi-step state prediction relative error $\frac{\|\mathscr{D}\mathbf{A}^k\mathscr{E}(\x_0)-\x_{k}\|_2}{\|\x_{k} \|_2}$ and the latent consistency relative error $\frac{\|\mathbf{A}^k\mathscr{E}(\x_0)-\mathscr{E}(\x_{k})\|_2}{\|\mathscr{E}(\x_{k}) \|_2}$. The results  indicate removing the latent consistency term significantly degrades the linear structure in the latent space, while removing the prediction term leads to reduced accuracy in the physical state. This confirms that the two terms play complementary roles in learning accurate and structured latent dynamics.}

\begin{figure}[htbp] 
	\centering  
		\begin{overpic}[width=0.8\linewidth, clip=true,tics=10]{ 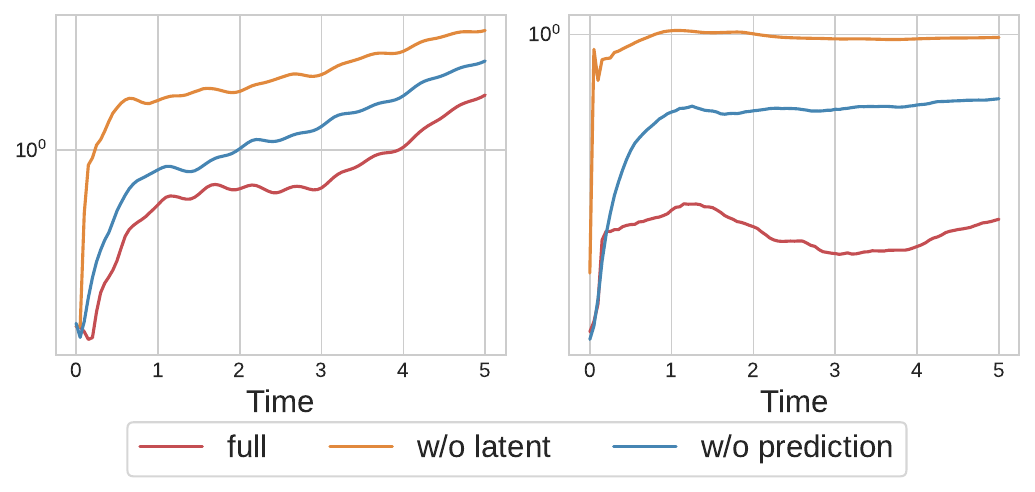}\end{overpic}
	\caption{\blue{Ablation study of the loss function \eqref{loss1} with latent dimension $n=20$. Left: relative state prediction error $\frac{\|\mathscr{D}\mathbf{A}^k\mathscr{E}(\x_0)-\x_{k}\|_2}{\|\x_{k} \|_2}$. Right: relative latent consistency error $\frac{\|\mathbf{A}^k\mathscr{E}(\x_0)-\mathscr{E}(\x_{k})\|_2}{\|\mathscr{E}(\x_{k}) \|_2}$. }}
	\label{44.10}
\end{figure}

To study the effect of the latent dimension, we train LAE models with latent dimensions $n\in\{16,20,24\}$ and apply the corresponding LAE-EnKF.  Each experiment is repeated over ten independent runs. Fig.~\ref{44.2} reports the time evolution of the relative RMSE, where solid lines denote the mean and shaded regions indicate $95\%$ confidence intervals. For all latent dimensions considered, LAE-EnKF achieves the lowest error and exhibits the most stable temporal behavior.


\begin{figure}[htbp] 
	\centering  
        \begin{overpic}[width=1\linewidth, clip=true,tics=10]{ 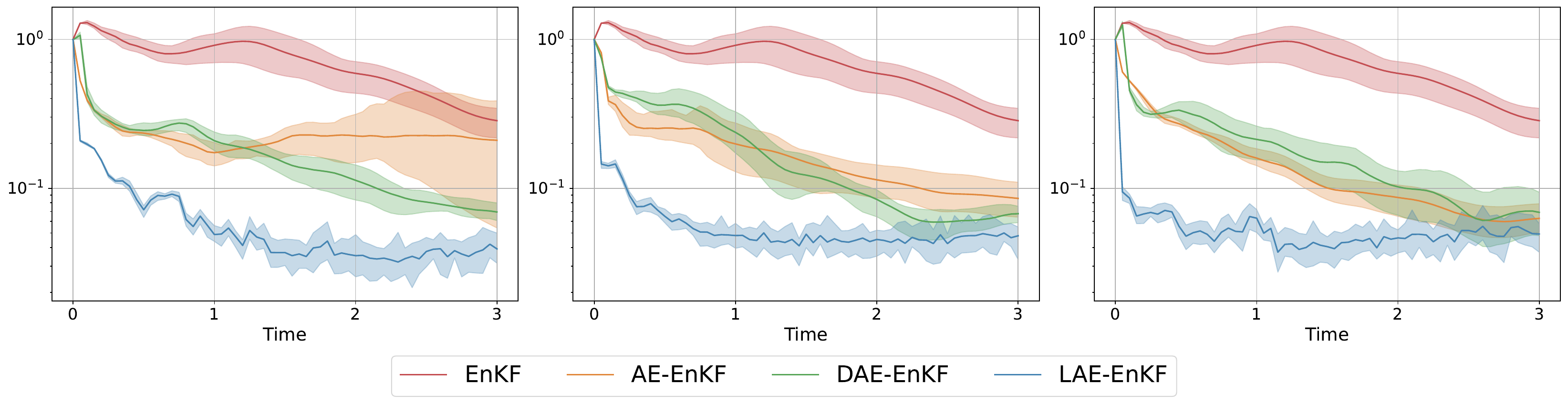}\put (15,25.95) {\footnotesize $n=16$ }\put (48,25.95) {\footnotesize $n=20$ }\put (81,25.95) {\footnotesize $n=24$ }\end{overpic}
	\caption{\blue{Time evolution of the relative RMSE for Example~\ref{ex2} under latent dimensions $n = 16,20, 24$. Solid curves indicate the mean over 10 independent runs, and the shaded regions represent the corresponding 95\% confidence intervals.}}
	\label{44.2}
\end{figure}

To further assess the quality of the reconstructed states, we examine the spatial structure of the assimilated solutions. Fig.~\ref{44.3} shows the reconstructed solution fields produced by LAE-EnKF at the final time $T$, together with the corresponding pointwise absolute errors, for different latent dimensions. Fig.~\ref{44.4} compares the reconstruction results of all methods for a representative case with $n=20$. In all cases, LAE-EnKF more accurately recovers the dominant spatial patterns and provides smoother reconstructions in regions without direct observations.

\begin{figure}[htbp] 
	\centering  
          \begin{overpic}[width=0.9\linewidth, clip=true,tics=10]{ 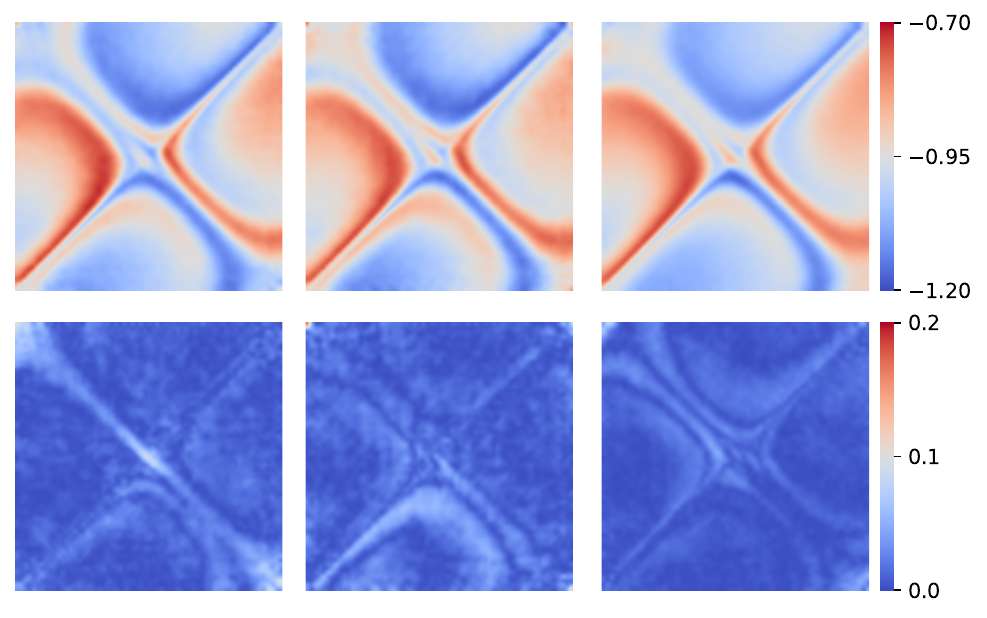}\put (13,61) {\footnotesize $n=16$ }\put (42,61) {\footnotesize $n=20$ }\put (72,61) {\footnotesize $n=24$ }\end{overpic}
    \caption{\blue{Reconstructed fields \red{(Top)} and corresponding pointwise absolute errors \red{(Bottom)} at the final time \red{$T=3$} by LAE-EnKF for different latent dimensions $n=16$, $n=20$, and $n=24$.}}
	\label{44.3}
\end{figure}

\begin{figure}[htbp] 
    \centering  
    \begin{overpic}[width=0.9\linewidth, clip=true,tics=10]{ 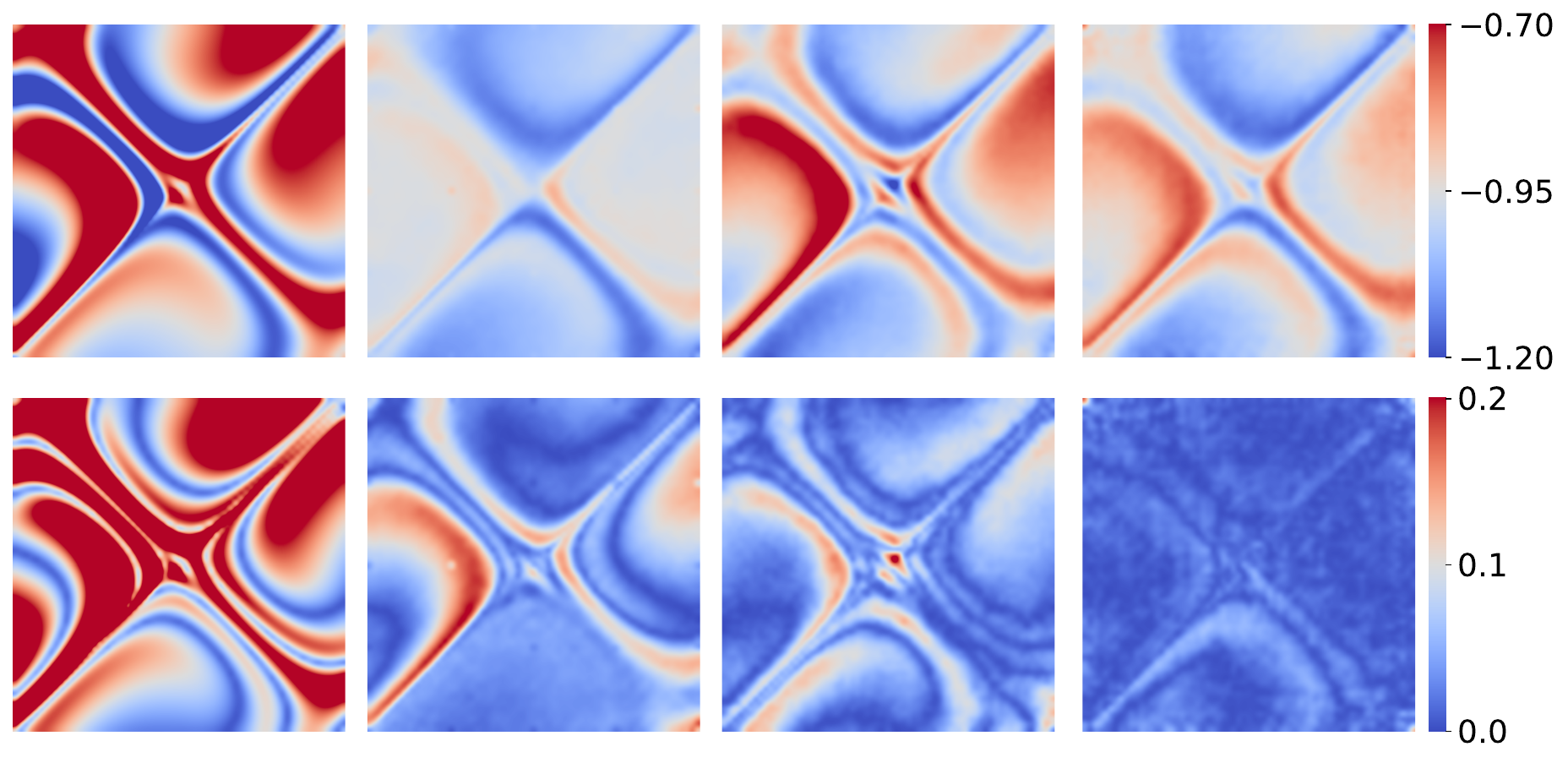}\put (9,48) {\footnotesize EnKF }\put (29,48) {\footnotesize AE-EnKF }\put (50,48) {\footnotesize DAE-EnKF }\put (74,48) {\footnotesize LAE-EnKF }\end{overpic}
    \caption{\blue{Comparison of reconstructed fields \red{(Top)} and corresponding pointwise absolute errors \red{(Bottom)} at the final time \red{$T=3$} for latent dimension $n=20$.}}
	\label{44.4}
\end{figure}

\blue{We next examine the effect of different latent observation operators. In addition to the choice $\mathbf{H}=\mathbf{I}$, we also consider both randomly generated observation operators and learned observation mappings in the latent space. For a fair comparison, the latent dynamical model obtained in the first training stage is kept fixed across all settings. Fig.~\ref{44.9} shows the time evolution of the relative RMSE under these different observation models, with the global error $\mathrm{E}_{\text{Rel},1:K}$ reported in the title. The results demonstrate that comparable performance is achieved across all observation operators, indicating that the proposed method is not sensitive to the specific choice of the observation model.}


\begin{figure}[htbp] 
\vspace{0.35cm}
	\centering  
        \begin{overpic}[width=1\linewidth, clip=true,tics=10]{ 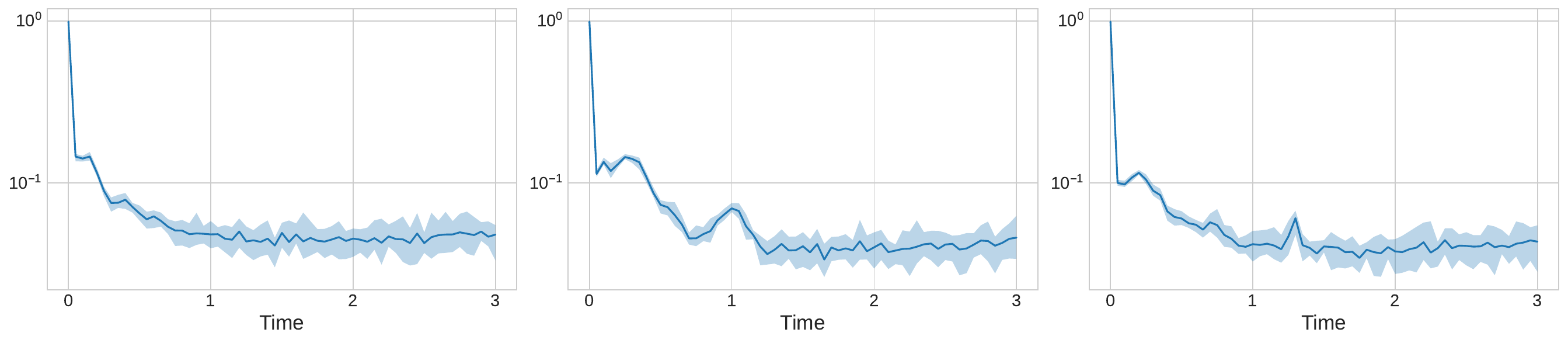}\blue{\put (15,24) {\footnotesize identity $\mathbf{H}$}\put (8,22) {\footnotesize ($\mathrm{E}_{\text{Rel},1:K}=0.0633$) }\put (45,24) {\footnotesize random $\mathbf{H}$}\put (41,22) {\footnotesize  ($\mathrm{E}_{\text{Rel},1:K}=0.0715$) }\put (80,24) {\footnotesize learned $\mathbf{H}$ }\put (75,22) {\footnotesize ($\mathrm{E}_{\text{Rel},1:K}=0.0556$)}}\end{overpic}
	\caption{\blue{Time evolution of the relative RMSE for Example \ref{ex2} under different latent observation operators: identity operator (Left), randomly operator (Middle), and learned operator (Right). Solid curves denote the mean over 10 independent runs, and shaded regions indicate the corresponding 95\% confidence intervals. The global relative error $\mathrm{E}_{\text{Rel},1:K}$  is reported in each subplot.
    }}
	\label{44.9}
\end{figure}

Finally, Tab.~\ref{44.5} summarizes the global relative RMSE $\mathrm{E}_{\mathrm{Rel},1:\red{K}}$ under different observation noise levels, together with the online computation time of the EnKF assimilation step. The results show that LAE-EnKF consistently achieves the best accuracy--efficiency trade-off. Compared with the standard EnKF, the proposed method substantially reduces computational cost while improving reconstruction accuracy. These results confirm that learning a stable linear latent representation and performing EnKF in the latent space can significantly enhance data assimilation performance for nonlinear PDE systems.

\begin{table}[tbp]
	\caption{Relative error {{$\mathrm{E}_{\text{Rel},1:\red{K}}$ }} and online computation time for Example 2.}
	\label{44.5}
	\centering
	\begin{tabular}{
			>{\centering\arraybackslash}m{1.5cm} 
			>{\centering\arraybackslash}m{1.9cm} 
			>{\centering\arraybackslash}m{2.1cm}
			>{\centering\arraybackslash}m{2.1cm}
			>{\centering\arraybackslash}m{2.1cm}
		}
		\toprule
		\toprule
		{$\boldsymbol{\sigma}$} &  \textbf{EnKF}   &  \textbf{AE-EnKF} & \textbf{DAE-EnKF}&  \textbf{LAE-EnKF}   \\ \midrule
		$\sigma=0.10$ & \blue{0.7987}   &  \blue{0.3371} & \blue{0.4369}  & \blue{\textbf{0.1058} }\\
		$\sigma=0.05$ & \blue{0.7553}   & \blue{0.2855} & \blue{0.3400}  &  \blue{\textbf{0.0780}} \\
		$\sigma=0.01$ & \blue{0.7516}   &  \blue{0.2344} & \blue{0.2860}  &  \blue{\textbf{0.0633}} \\
		\midrule
		\textbf{Time (s)} & \blue{2837.29}  & \blue{ 6.75} & \blue{0.10} &  \blue{\textbf{0.05}}\\
		\bottomrule 
	\end{tabular}
\end{table}

\subsection{Lorenz--96 model}

We consider the chaotic Lorenz--96 system,
\begin{equation}
	\frac{d x_i}{dt}
	= (x_{i+1}-x_{i-2})\,x_{i-1} - x_i + F,
	\qquad i=0,\ldots,D-1,
\end{equation}
with periodic boundary conditions
$x_D\equiv x_0$, $x_{-1}\equiv x_{D-1}$, and $x_{-2}\equiv x_{D-2}$.
In all experiments, we set $D=40$ and the forcing parameter $F=8$, which is a standard configuration that produces sustained chaotic dynamics. The system is integrated using a fourth-order Runge--Kutta scheme with time step $0.01$.

Observations are generated according to
\[
\mathbf{y}_k = \mathcal{H}\bigl(\mathbf{x}(t_k)\bigr) + \boldsymbol{\eta}_k,
\qquad
\boldsymbol{\eta}_k \sim \mathcal{N}(\mathbf{0}, \sigma^2 \mathbf{I}),
\]
\blue{where $\mathcal{H}$ denotes the observation operator. We consider both linear and nonlinear observation settings. In the linear observation setting, the observation operator $\mathcal{H}$ is chosen as a subsampling operator that extracts every other state variable,}
\[
\mathcal{H}(\mathbf{x}) = (x_0, x_2, \ldots, x_{38})^\top,
\]
\blue{with observation noise level $\sigma = 1$.} This setup represents a partially observed scenario in which only half of the state variables are directly measured.

\paragraph{Dense observation regime}
We first consider a relatively dense observation setting with observations available every $\Delta t = 0.1$. LAE models are trained with latent dimensions $n \in \{64,80,100\}$, and the resulting LAE-EnKF is compared with AE-EnKF, DAE-EnKF, and an EnKF baseline. \red{Tab.~\ref{43.1} reports the global relative RMSE $\mathrm{E}_{\mathrm{Rel},1:K}$ over the full state for all methods, both with and without localization.} \red{Localization significantly improves the performance of EnKF and AE-EnKF, consistent with the well-known limitation of EnKF for the Lorenz--96 system. Without localization, sampling errors in the empirical covariance introduce spurious long-range correlations, leading to instability and large reconstruction errors. Since EnKF and AE-EnKF operate in the original state space, localization is essential for stabilizing their performance. Accordingly, unless otherwise specified, all references to EnKF and AE-EnKF hereafter correspond to their localized implementations. In contrast, DAE-EnKF and LAE-EnKF operate in a latent space, where spurious correlations are inherently reduced. Localization provides negligible improvement for these methods. In particular, LAE-EnKF achieves low errors, with nearly identical performance with and without localization, indicating that localization is unnecessary in this setting. This indicates that latent representations mitigate sampling errors through structured dynamics.}

\begin{table}[tbp]
\caption{\red{Relative RMSE $\mathrm{E}_{\mathrm{Rel},1:K}$ for different methods with varying latent dimension $n$, reported without (with) covariance localization, under two observation intervals.}}
\label{43.1}
\cred
\centering
\begin{tabular}{c c c c c c}
\toprule
$\Delta t$&
$\boldsymbol{n}$ 
& \textbf{EnKF }
& \textbf{AE-EnKF} 
& \textbf{DAE-EnKF} 
& \textbf{LAE-EnKF}  \\
\midrule
\midrule	
 & --     & 0.5977 (0.1590) & --          & --          & --       \\
&$64$ & --      & 0.4504 (0.2027) & 0.8276 (0.8195) & \textbf{0.1475} (0.1480)  \\
$0.1$&$80$ & --  & 0.3793 (0.2026) & 0.8202 (0.8697) & \textbf{0.1509} (0.1512)  \\
&$100$& --  & 0.3563 (0.2080) & 0.6696 (0.6804) & \textbf{0.1426} (0.1431)  \\
\midrule	
& --     & 0.3223 (0.1689) & --          & --          & --       \\
&$64$ & -- & 0.4376 (0.2416) & 1.1994 (1.0393) & \textbf{0.1862} (0.1864)  \\
$0.2$&$80$ & -- & 0.3502 (0.2334) & 0.7244 (0.7274) & \textbf{0.1704} (0.1708)  \\
&$100$& -- & 0.2723 (0.2360) & 0.6567 (0.6452) & \textbf{0.1679} (0.1682)  \\
\bottomrule
\end{tabular}
\end{table}

Fig.~\ref{43.2} shows the reconstructed spatial fields and the corresponding errors for the representative case $n=64$, while Fig.~\ref{43.3} presents the temporal evolution of one observed and one unobserved state variable. DAE-EnKF exhibits larger reconstruction errors and irregular temporal behavior, indicating instability in the learned latent dynamics. In contrast, LAE-EnKF consistently achieves lower errors and smoother temporal evolution for both observed and unobserved variables. This demonstrates that enforcing stable linear dynamics in the latent space substantially improves the recovery of hidden states in a chaotic system.

\begin{figure}[tbp] 
	\centering  
         \begin{overpic}[width=0.9\linewidth, clip=true,tics=10]{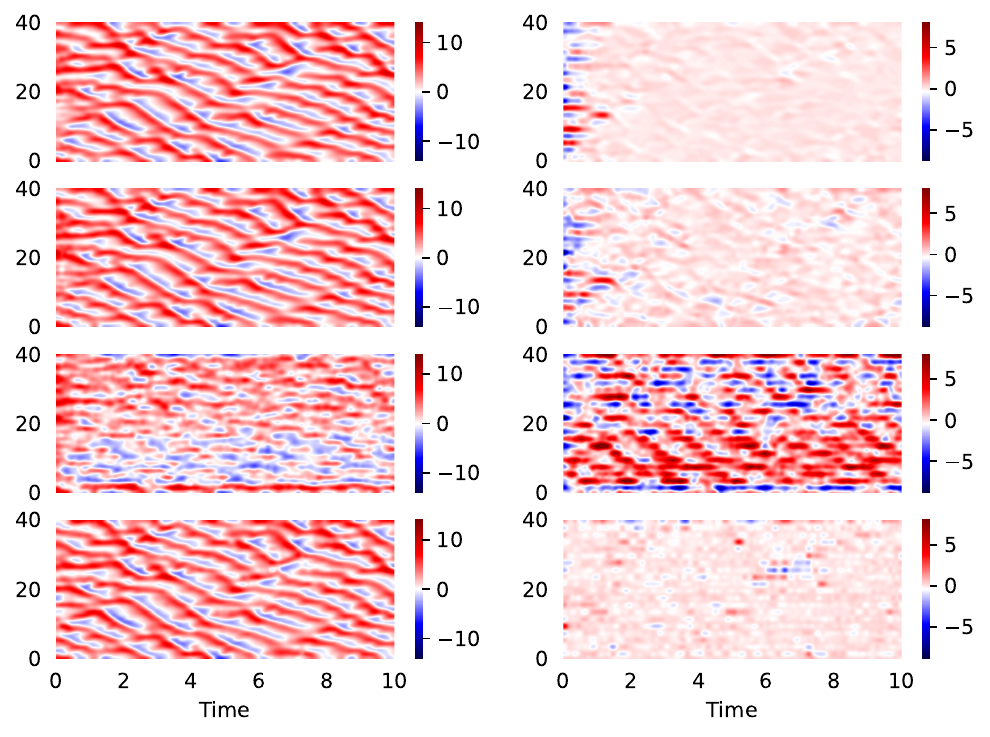}\end{overpic}
	\caption{Reconstructed spatial state fields (Left) and corresponding error (Right) with the observation interval $\Delta t = 0.1$ and latent dimension $n=64$. From top to bottom, the results show the localization EnKF, \red{localization AE-EnKF}, DAE-EnKF and the LAE-EnKF.	}
	\label{43.2}
\end{figure}

\begin{figure}[htbp] 
\centering  
    \begin{overpic}[width=0.9\linewidth, clip=true, tics=10]{ 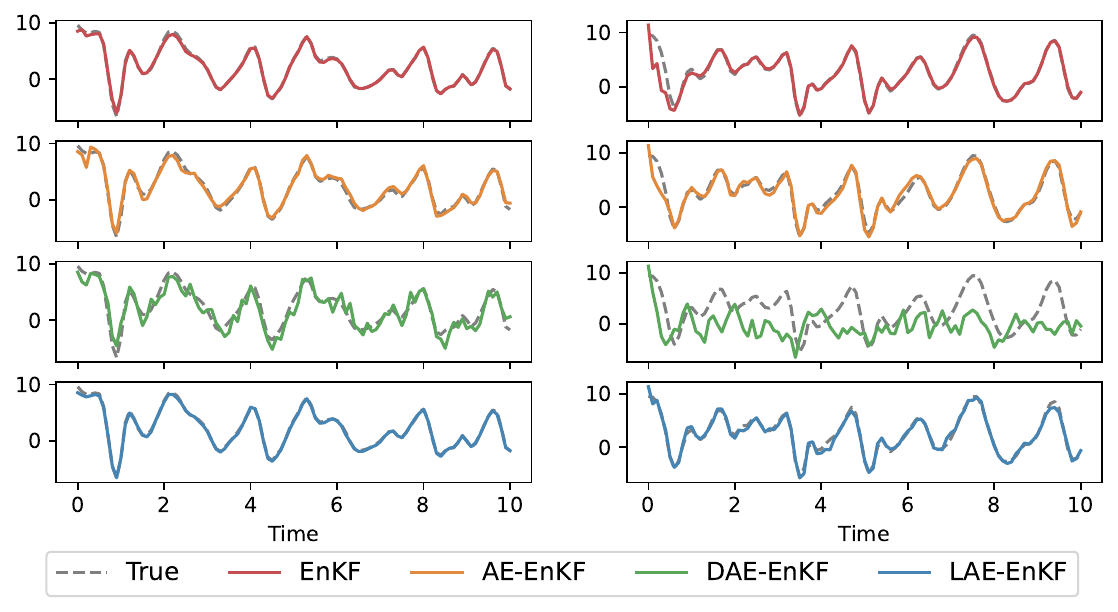}
	\end{overpic}
\caption{Time evolution of a representative observed variable (Left) and an unobserved variable (Right) reconstructed with observation interval $\Delta t = 0.1$ and latent dimension $n=64$, compared with the corresponding true trajectories. From top to bottom, the results show the localization EnKF, \red{localization AE-EnKF}, DAE-EnKF and the LAE-EnKF.	}
\label{43.3}
\end{figure}

\paragraph{Sparse observation regime}
We next consider a sparser observation setting with the observation interval increased to $\Delta t = 0.2$. The same latent dimensions $n\in\{64,80,100\}$ are used for all learned-model-based methods to ensure a fair comparison. The corresponding error metrics are again reported in Tab.~\ref{43.1}. Fig.~\ref{43.5} shows the reconstructed spatial fields and errors for $n=64$. Despite the reduced observation frequency, LAE-EnKF maintains reconstruction accuracy comparable to that obtained with $\Delta t=0.1$, demonstrating strong robustness to observation sparsity.

\begin{figure}[htbp] 
	\centering  
        \begin{overpic}[width=0.9\linewidth, clip=true, tics=10]{ 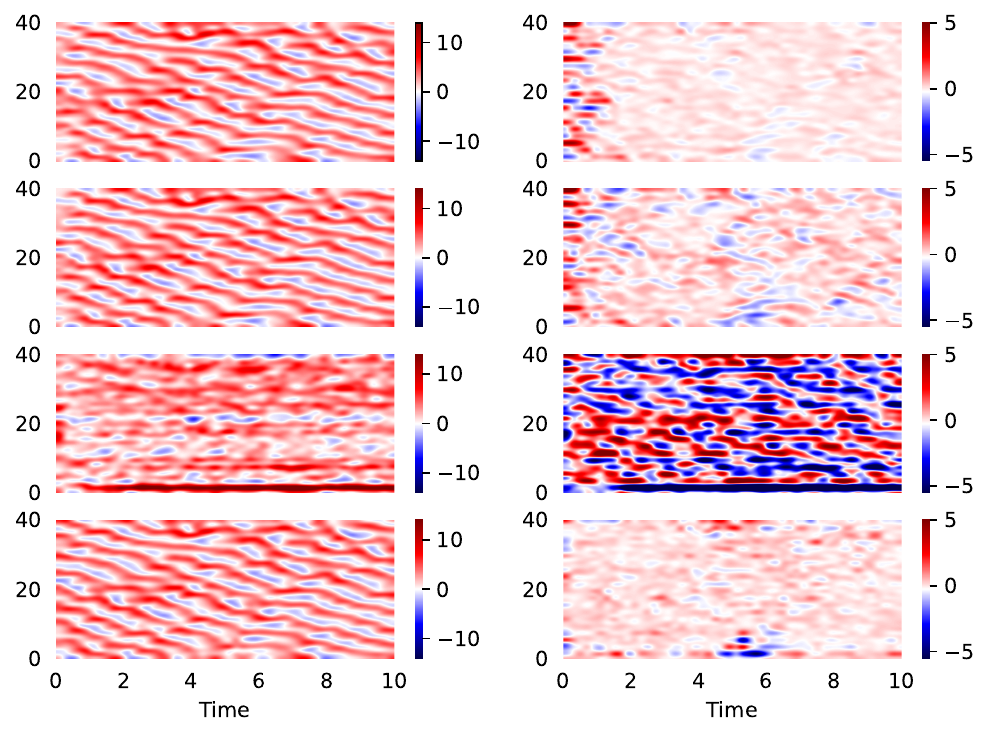}
		\end{overpic}
	\caption{Reconstructed spatial state fields (Left) and corresponding error (Right) with the observation interval $\Delta t = 0.2$ and latent dimension $n=64$. From top to bottom, the results show the localization EnKF, \red{localization AE-EnKF}, DAE-EnKF and the LAE-EnKF.}
	\label{43.5}
\end{figure}

Fig.~\ref{43.6} shows the temporal evolution of one observed and one unobserved variable reconstructed by LAE-EnKF. As the latent dimension increases, the reconstruction of the observed variable improves steadily, while the unobserved variable is accurately captured across all tested dimensions. This behavior indicates that the learned linear latent dynamics provide reliable temporal extrapolation and support stable assimilation under infrequent observations.

\begin{figure}[tbp] 
	\centering  
		\begin{overpic}[width=0.9\linewidth, clip=true,tics=10]{ 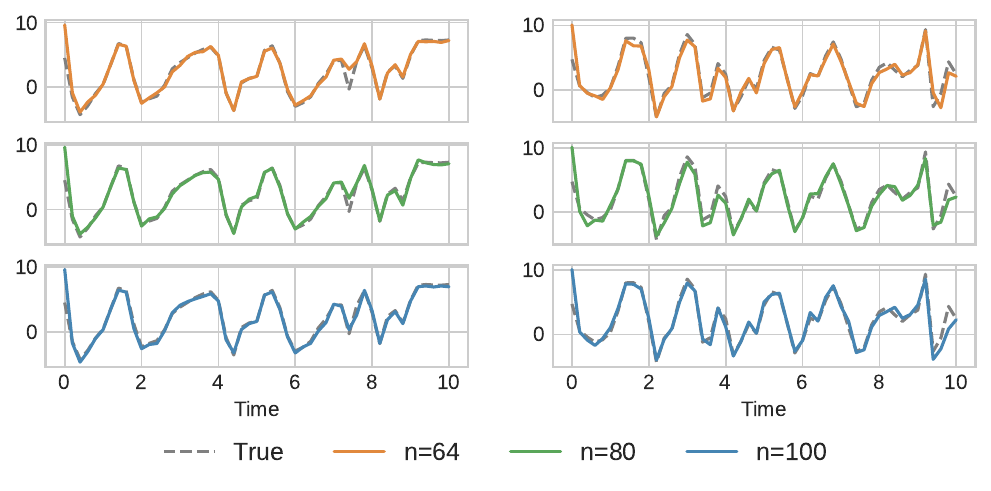}\end{overpic}
	\caption{Time evolution of a representative observed variable (Left) and an unobserved variable (Right) reconstructed by LAE-EnKF with observation interval $\Delta t = 0.2$ for different latent dimension $n=64,80,100$. }
	\label{43.6}
\end{figure}


\paragraph{\blue{Nonlinear observation setting}}
\blue{We further consider a nonlinear observation operator defined as
\[
\mathcal{H}(\mathbf{x}) = \arctan(\mathbf{x}).
\]
Under nonlinear observation settings, we compare the localized EnKF with the proposed LAE-EnKF. For LAE-EnKF, we investigate the impact of different latent observation operators $\mathbf{H}$, including the identity operator, random operators, and learned operators. Fig.~\ref{43.3} shows the reconstructed trajectory of a representative state component, with the corresponding global RMSE $\mathrm{E}_{\mathrm{Rel},1:K}$ annotated above each figure. The results show that the performance of the standard EnKF deteriorates under nonlinear observations, whereas LAE-EnKF remains stable and achieves more accurate state reconstruction, closely matching the true trajectory. Moreover, structured choices of $\mathbf{H}$, such as the identity or random operators, generally outperform learned operators in terms of reconstruction accuracy. This behavior may be attributed to the fact that structured operators are more likely to preserve the geometric and covariance structure of the latent space, while a fully learned $\mathbf{H}$ may introduce distortions in dominant dynamical directions, thereby affecting the effectiveness of the Kalman update.}

\begin{figure}[htbp] 
    \vspace{0.1cm}
	\centering  
		\begin{overpic}[width=0.96\linewidth, clip=true,tics=10]{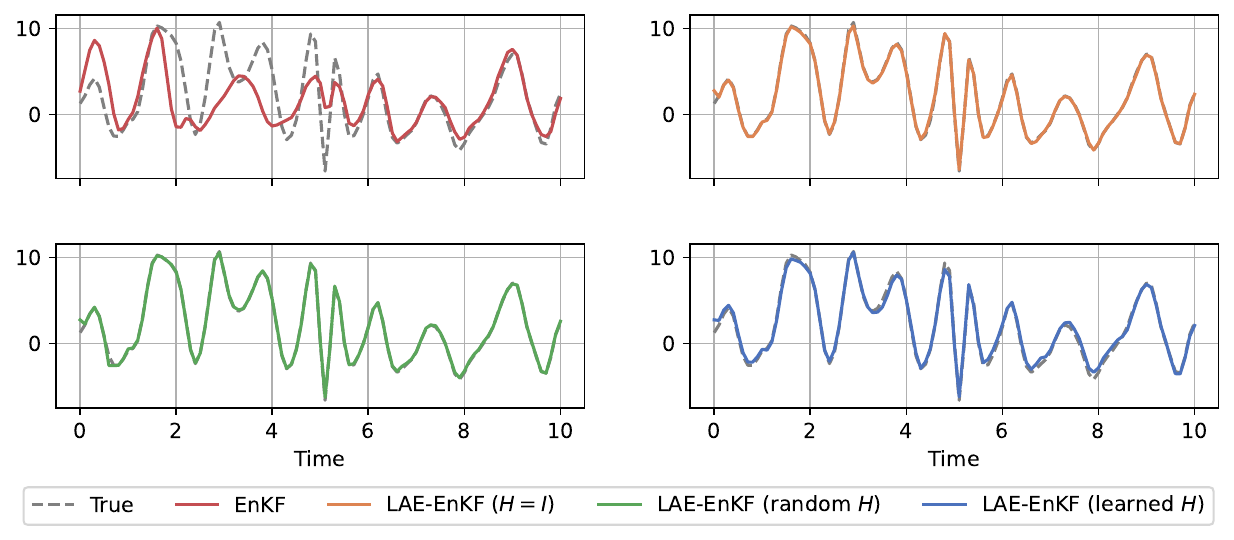}\blue{\put(17,44){\footnotesize $\mathrm{E}_{\text{Rel},1:K}=0.6868$}\put(67,44){\footnotesize $\mathrm{E}_{\text{Rel},1:K}=0.0374$} \put(17,25.5){\footnotesize $\mathrm{E}_{\text{Rel},1:K}=0.0306$} \put(67,25.5){\footnotesize $\mathrm{E}_{\text{Rel},1:K}=0.0990 $}}
        \end{overpic}

    \caption{\blue{Reconstruction of representative state components under nonlinear observations. Comparison between localized EnKF and LAE-EnKF under different latent observation operators, including the identity, random, and learned operators. The global RMSE $\mathrm{E}_{\text{Rel},1:K}$ is reported above each panel.}}
	\label{43.4}
\end{figure}

Overall, the Lorenz–96 experiments show that the proposed LAE-EnKF captures the essential dynamics of a chaotic system within a compact linear latent representation. By enabling stable and accurate data assimilation under partial and sparse observations, the LAE-EnKF extends the applicability of ensemble Kalman filtering to nonlinear chaotic systems in a fully data-driven setting.

\section{Conclusion}\label{s6}
This paper introduced the latent autoencoder ensemble Kalman filter (LAE-EnKF), a structure-preserving framework for data assimilation in nonlinear and partially observed systems. By learning a nonlinear encoder–decoder pair together with a stable linear latent dynamical model and a unified observation embedding, the proposed approach constructs a latent representation in which the assumptions underlying Kalman filtering are approximately satisfied. This design enables ensemble Kalman filtering to be performed directly in the learned latent space, mitigating the structural mismatch that often limits the performance of standard EnKF in strongly nonlinear regimes. Extensive experiments on representative nonlinear and chaotic systems showed that the LAE-EnKF consistently improves assimilation accuracy, robustness, and long-term stability compared with the classical EnKF and existing autoencoder-assisted methods, while maintaining favorable computational efficiency.

The proposed framework opens several promising directions for future research. These include extensions to systems with model error and parameter uncertainty, adaptive latent dimension selection, and integration with localization and inflation strategies for large-scale geophysical applications. More broadly, this work highlights the potential of structure-aware representation learning as a principled pathway for advancing ensemble-based data assimilation in complex nonlinear systems.

\appendix

\section{Proof of Theorem \ref{thm1}}\label{appendixA}
To prove \cref{thm1}, we follow the proof strategy developed in \cite{LIU2024DAE}, which establishes covering number bounds for deep ReLU networks when approximating functions supported on low-dimensional manifolds. Our proof adopts the same notation and technical framework, with adaptations tailored to the present setting.

We denote by $\mathcal{F}_{\mathrm{NN}}(d_1,d_2;L,p,K,\kappa,R)$ the class of neural networks introduced in \cite{LIU2024DAE}, where $d_1$ and $d_2$ are the input and output dimensions. The networks have depth $L$, width $p$, and at most $K$ nonzero parameters, each bounded by $\kappa$ in magnitude. The effective parameter count $K$ may depend on $(d_1,d_2)$. For inputs satisfying $\|\mathbf{x}\|_\infty \le B$, the network outputs are uniformly bounded by $R$. The proof is based on controlling both approximation and estimation errors through a covering number bound for the class $\mathcal{F}_{\mathrm{NN}}$ under the uniform norm. To make the argument precise, we first introduce the assumptions and auxiliary lemmas needed for the subsequent analysis. \red{For the subsequent analysis, we use the mixed $L^{\infty,\infty}$ norm defined by $ \|\mathbf{f}\|_{L^{\infty,\infty}} := \sup_{\mathbf{x} \in \Omega} \|\mathbf{f}(\mathbf{x})\|_\infty$. The latent variables are assumed to lie in a bounded domain, taken as $[-\Lambda,\Lambda]^n$ up to approximation of the manifold parametrization $\varphi$.}
 
\begin{definition}[Covering number {\cite[Definition 2.1.5]{Vaart1996cover}}]
For any $\delta>0$, the covering number of $\mathcal{F}$ is defined as
\[
\mathcal{Q}(\delta,\mathcal{F},\|\cdot\|) = \min\left\{  |\red{\mathcal{C}}|:\red{\mathcal{C}}\text{ is a }\delta\text{-cover of }\mathcal{F}\text{ under }\|\cdot\|\right\},
\]
\end{definition}

\begin{lemma}[{\cite[Lemma 5.3]{Chen2022cover}}]\label{lemma_nncover}
Assume that the input domain satisfies $\|\mathbf{x}\|_\infty\le B$. For any $\delta>0$,
\begin{align}\label{NNcover}
    \mathcal{Q}\bigl(\delta,\mathcal{F}_{\mathrm{NN}}(d_1,d_2;L,p,K,\kappa,R),\|\cdot\|_{L^{\infty,\infty}}\bigr) \le \left(\frac{2L^2(pB+2)\kappa^L p^{L+1}}{\delta}\right)^K .
\end{align}
\end{lemma}

We also define the class of linear latent dynamics:
\[
	\mathcal{F}_{\mathrm{Lin}}^{A} :=  \bigl\{z\mapsto \mathbf{A}z\mid \mathbf{A}\in\mathbb{R}^{n\times n},\ \|\mathbf{A}\|_{2}\le \rho\bigr\},
\]
where $0<\rho\leq 1$ is a fixed constant.
	
\begin{lemma}\label{lem_EAD_approx}
Suppose Assumption~\ref{assum1} holds. For any $0<\varepsilon<1$, there exist $\widetilde{\mathscr E}\in\mathcal{F}_{\mathrm{NN}}^{\mathscr{E}} = \mathcal{F}\left(D, n ; L_{\mathscr{E}}, p_{\mathscr{E}}, K_{\mathscr{E}}, \kappa_{\mathscr{E}}, R_{\mathscr{E}}\right)$, $\widetilde{\mathbf A}\in\mathbb{R}^{n\times n}$ with $\|\widetilde{\mathbf A}\|_2\le\rho \red{\ (\rho \le 1)}$ and $\widetilde{\mathscr D}\in \mathcal{F}_{\mathrm{NN}}^{\mathscr{D}} = \mathcal{F}\left(n, D;L_{\mathscr{D}}, p_{\mathscr{D}}, K_{\mathscr{D}}, \kappa_{\mathscr{D}}, R_{\mathscr{D}}\right)$ such that
	\begin{equation}\label{eq:LAE_uniform_approx}
		\sup_{\mathbf{x}\in\mathcal{M}}
		\big\| \widetilde{\mathscr D}\!\left(\widetilde{\mathbf A}\widetilde{\mathscr E}(\mathbf{x})\right) -\mathbf F(\mathbf{x}) \big\|_\infty \le \varepsilon.
		\end{equation}
	\end{lemma}
\begin{proof}
	By \cite[Theorem~2.2]{CLONINGER2021encoder}, for any $0<\varepsilon<1$, let $\varepsilon_1 =\frac{\varepsilon}{(1+C_\phi\rho\sqrt n)}<1$ , there exists an encoder network $\widetilde{\mathscr E}$ such that
	\[
	\sup_{\mathbf{x}\in\mathcal M}\|\widetilde{\mathscr E}(\mathbf{x})-\varphi(\mathbf{x})\|_\infty \le \varepsilon_1.
	\]
	Moreover, by Kirszbraun's theorem \cite{Kirszbraun1934}, the inverse map $\phi=\varphi^{-1}$ admits a Lipschitz extension (still denoted by $\phi$) to $[-\Lambda,\Lambda]^{n}$ with Lipschitz constant $C_\phi$. Applying the approximation result of \cite[Lemma 8]{YAROTSKY2017decoder} to this extended $\phi$, there exists a decoder network $\widetilde{\mathscr D}$ such that
	\[
	\sup_{\mathbf{z}\in[-\Lambda,\Lambda]^n}\|\widetilde{\mathscr D}(\mathbf{z})-\phi(\mathbf{z})\|_\infty \le \varepsilon_1.
	\]
	Set $\widetilde{\mathbf A}=\mathbf A^\star$. For any $\mathbf{x}\in\mathcal M$, letting $u=\mathbf A^\star\widetilde{\mathscr E}(\mathbf{x})$ and $v=\mathbf A^\star \varphi(\mathbf{x})$, we have
	\[
	\|\widetilde{\mathscr D}(u)-\mathbf F(\mathbf{x})\|_\infty =\|\widetilde{\mathscr D}(u)-\phi(v)\|_\infty \le \|\widetilde{\mathscr D}(u)-\phi(u)\|_\infty + \|\phi(u)-\phi(v)\|_\infty .
	\]
	Using Lipschitz continuity of $\phi$, we obtain
	\[
	\|\phi(u)-\phi(v)\|_\infty \le \|\phi(u)-\phi(v)\|_2 \le C_\phi\|\mathbf A^\star\|_2\,\|\widetilde{\mathscr E}(x)-\varphi(x)\|_2 \le C_\phi\rho\sqrt n\,\varepsilon_1.
	\]
	Hence
	\[\sup\limits_{\mathbf{x}\in\mathcal M}\|\widetilde{\mathscr D}(\mathbf A^\star \widetilde{\mathscr E} (\mathbf{x})) - \mathbf F(\mathbf{x})\|_\infty \le \varepsilon.\]
	\end{proof}

	We now present the proof of Theorem~ \ref{thm1}.
	\begin{proof}[Proof of Theorem \ref{thm1}]
		To simplicity the notations, for any given $\mathcal{F}_{\mathrm{NN}}^{\mathscr{E}}$, $\mathcal{F}_{\mathrm{Lin}}^{{A}}$ and $\mathcal{F}_{\mathrm{NN}}^{\mathscr{D}}$, we define the network class
		\begin{align*}
			\mathcal{F}_{\mathrm{NN}}^{\mathscr{G}} &= \{\mathscr{G}=\mathscr{D}\circ {\mathbf{A}}\circ {\mathscr{E}}\mid  \mathscr{D}\in \mathcal{F}_{\mathrm{NN}}^{\mathscr{D}},\ \mathbf{A}\in \mathcal{F}_{\mathrm{Lin}}^{{A}},\ \mathscr{E}\in \mathcal{F}_{\mathrm{NN}}^{\mathscr{E}} \}~,   \\
			\mathcal{F}_{\mathrm{AE}}^{\mathscr{H}} &= \{\mathscr{H}= {\mathbf{A}}\circ {\mathscr{E}}\mid  \mathbf{A}\in \mathcal{F}_{\mathrm{Lin}}^{{A}},\ \mathscr{E}\in \mathcal{F}_{\mathrm{NN}}^{\mathscr{E}} \} ~, \\
			\mathcal{F}_{\mathrm{NN}}^{\mathscr{D}\mathscr{E}} &= \{\mathscr{D}\mathscr{E}= \{\mathscr{D}\circ {\mathscr{E}}\mid  \mathscr{D}\in \mathcal{F}_{\mathrm{NN}}^{\mathscr{D}},\ \mathscr{E}\in \mathcal{F}_{\mathrm{NN}}^{\mathscr{E}} \}\} ~,
		\end{align*}
		and denote $\widehat{\mathscr{G}}=\widehat{\mathscr{D}}\circ \widehat{\mathbf{A}}\circ \widehat{\mathscr{E}}$, and $\widehat{\mathscr{H}}= \widehat{\mathbf{A}}\circ \widehat{\mathscr{E}}$. 
		
		We decompose the generalization error into a empirical error term and a estimation error term
    \begin{align*}  \red{\mathbb{E}_{\mathcal{S}}\mathbb{E}_{(\mathbf{x},\mathbf{x}^+)\sim\mathbb{P}_{X,X^+}} }& \red{ \Bigl[  \| \widehat{\mathscr{G}}(\mathbf{x}) - \mathbf{x}^+ \|_2^2 + \| \widehat{\mathscr{D}}\circ \widehat{\mathscr{E}}(\mathbf{x}) - \mathbf{x}\|_2^2} 
    \red{+ \| \widehat{\mathscr{H}}(\mathbf{x}) - \widehat{\mathscr{E}}(\mathbf{x}^+)\|_2^2 + R(\widehat{\mathbf{A}}) \Bigr]} \\
     \red{=}& \red{\chi_1 + \chi_2~.}
    \end{align*}
	where
    \begin{align*}
        \red{\chi_1 =}& \, \red{C_1\,\mathbb{E}_{\mathcal{S}} \Bigl[ \frac{1}{N}\sum_{i=1}^N \Bigl(   \| \widehat{\mathscr{G}}(\mathbf{x}_i) -\mathbf{x}_i^+\|_2^2 + \| \widehat{\mathscr{D}}\circ \widehat{\mathscr{E}}(\mathbf{x}_i) - \mathbf{x}_i\|_2^2 +\, \| \widehat{\mathscr{H}}(\mathbf{x}_i) - \widehat{\mathscr{E}}(\mathbf{x}_i^+)\|_2^2 \Bigr) \Bigr] }\\
        &\red{+\, R(\widehat{\mathbf{A}}), }\\
		\red{\chi_2  =}&\, \red{\mathbb{E}_{\mathcal{S}}\mathbb{E}_{(\x,\x^+)\sim \mathbb{P}_{X,X^+}} \| \widehat{\mathscr{G}}(\mathbf{x}) - \mathbf{x}^+\|_2^2 - C_1\,\mathbb{E}_{\mathcal{S}}\left[ \frac{1}{N}\sum_{i=1}^N\| \widehat{\mathscr{G}}(\mathbf{x}_i) - \mathbf{x}_i^+\|_2^2 \right] }\\
		&\red{ +\,  \mathbb{E}_{\mathcal{S}}\mathbb{E}_{(\x,\x^+)\sim \mathbb{P}_{X,X^+}} \| \widehat{\mathscr{D}}\circ \widehat{\mathscr{E}}(\mathbf{x}) - \mathbf{x}\|_2^2 - C_1\,\mathbb{E}_{\mathcal{S}}\left[ \frac{1}{N}\sum_{i=1}^N  \| \widehat{\mathscr{D}}\circ \widehat{\mathscr{E}}(\mathbf{x}_i) - \mathbf{x}_i\|_2^2 \right]} \\
		& \red{+\, \mathbb{E}_{\mathcal{S}}\mathbb{E}_{(\x,\x^+)\sim \mathbb{P}_{X,X^+}} \| \widehat{\mathscr{H}}(\mathbf{x}) - \widehat{\mathscr{E}}(\mathbf{x}^+)\|_2^2- \, C_1\,\mathbb{E}_{\mathcal{S}}\left[ \frac{1}{N}\sum_{i=1}^N \| \widehat{\mathscr{H}}(\mathbf{x}_i) - \widehat{\mathscr{E}}(\mathbf{x}_i^+) \|_2^2\right].}	
		\end{align*}
		Let $\widetilde{\mathscr{E}}$, and $\widetilde{\mathscr{D}}$ be the networks in Lemma \ref{lem_EAD_approx} with accuracy $\varepsilon_1 =\frac{\varepsilon}{1+C_\phi\rho\sqrt n}<1$.  For the first term $\red{C_1}\mathbb{E}_{\mathcal{S}}\left[\frac{1}{N}\sum\limits_{i=1}^N\| \widehat{\mathscr{G}}(\mathbf{x}_i) - \red{\mathbf{x}_i^+} \|_2^2 \right]$ of $\chi_1$, we have
		\begin{align}\label{chi1_1_result}
			& \red{C_1}\,\mathbb{E}_{\mathcal{S}}\left[\frac{1}{N}\sum_{i=1}^N\| \widehat{\mathscr{G}}(\mathbf{x}_i) - \red{\mathbf{x}_i^+}\|_2^2 \right] \notag\\
            \le &\ \red{C_1}\, \mathbb{E}_{\mathcal{S}}\inf_{\mathscr{G}\in\mathcal{F}_{NN}^{\mathscr{G}}} \left[\frac{1}{N}\sum_{i=1}^N\| \mathscr{G}(\mathbf{x}_i) - \mathbf{F}(\mathbf{x}_i)-\red{\bfw_i}\|_2^2 \right] \notag\\
			\leq &\ \red{C_1}\inf_{\mathscr{G}\in\mathcal{F}_{NN}^{\mathscr{G}}} \mathbb{E}_{\mathcal{S}} \left[\frac{1}{N}\sum_{i=1}^N\| \mathscr{G}(\mathbf{x}_i) - \mathbf{F}(\mathbf{x}_i)-\red{\bfw_i}\|_2^2 \right] \notag\\
			 \leq &\ \red{C_1}\, \mathbb{E}_{(\x,\x^+)\sim \mathbb{P}_{X,X^+}} \left[\|\widetilde{\mathscr{D}} \circ{\mathbf{A}^\star} \circ \widetilde{\mathscr{E}}(\mathbf{x})-\mathbf{F}(\mathbf{x})-\red{\bfw}\|_2^2\right] \notag\\
			 \leq & \ \red{C_1}\, D \sup _{\mathbf{x} \in \mathcal{M}}\|\widetilde{\mathscr{D}} \circ {\mathbf{A}^\star} \circ \widetilde{\mathscr{E}}(\mathbf{x})-\mathbf{F}(\mathbf{x})\|_{\infty}^2 \notag+\red{C_1 \E\|\bfw\|_2^2}\\
			 \leq &\ \red{C_1}\, D\, \varepsilon^2+\red{C_1\, \E\|\bfw\|_2^2} .
		\end{align}
		
		Similarly, for the second term of $\chi_1$, let $\varepsilon_1 =\min\{\frac{\varepsilon}{1+C_\phi\rho\sqrt n},\ \frac{\varepsilon}{1+C_\phi \sqrt n} \}$, we can obtain 
		\begin{align}\label{chi1_2_result}
			\red{C_1}\, \mathbb{E}_{\mathcal{S}}\left[ \frac{1}{N}\sum_{i=1}^N 	\| \widehat{\mathscr{D}}\circ \widehat{\mathscr{E}}(\mathbf{x}_i) - \mathbf{x}_i\|_2^2 \right] \leq  \red{C_1}\, D\, \varepsilon^2 .
		\end{align}
		For the last term of $\chi_1$, we can also derive the upper bound
		\begin{align*}
			& \ \red{C_1}\, \mathbb{E}_{\mathcal{S}}\left[\frac{1}{N}\sum_{i=1}^N\| \widehat{\mathscr{H}}(\mathbf{x}_i) - \widehat{\mathscr{E}}(\red{\mathbf{x}_i^+})\|_2^2 \right]\\
			 \le&\  \red{C_1}\inf_{\mathscr{H}\in\mathcal{F}_{AE}^{\mathscr{H}}} \mathbb{E}_{(\x,\x^+)\sim \mathbb{P}_{X,X^+}} \left[\| \mathscr{H}(\mathbf{x}) -{\mathscr{E}}\left(\mathbf{F}(\mathbf{x})+\red{\bfw}\right)\|_2^2 \right]\\
        \le&\  \red{C_1\inf_{\mathscr{H}\in\mathcal{F}_{AE}^{\mathscr{H}}} \mathbb{E}_{(\x,\x^+)\sim \mathbb{P}_{X,X^+}}\bigg\{ \frac{C_1}{C_1-1}\left[\| \mathscr{H}(\mathbf{x}) -{\mathscr{E}}(\mathbf{F}(\mathbf{x}))\|_2^2 \right]}\\
        &\qquad \qquad\qquad \qquad\qquad\qquad\red{+\, C_1\|{\mathscr{E}}(\mathbf{F}(\mathbf{x}))-\mathscr{E}(\mathbf{F}(\mathbf{x})+\red{\bfw})\|_2^2\bigg\}.}
		\end{align*}
		We next bound the first row of last step by evaluating it at the candidate $\widetilde{\mathscr{H}} := \widetilde{\mathbf A}\,\widetilde{\mathscr E}$. For any $\mathbf{x}\in\mathcal M$, by the triangle inequality,
		\begin{align*}
			\| \widetilde{\mathscr{H}}(\mathbf{x}) -\widetilde{\mathscr{E}}(\mathbf{F}(\mathbf{x}))\|_2 \le& \| {\mathbf{A}^\star}\widetilde{\mathscr{E}}(\mathbf{x}) - {\mathbf{A}^\star}\varphi(\mathbf{x})\|_2  + \|{\mathbf{A}^\star}\varphi(\mathbf{x}) -\widetilde{\mathscr{E}}(\mathbf{F}(\mathbf{x}))\|_2\\
			\le&  (\rho+1)\sqrt{n} \varepsilon_1 ~.
		\end{align*}

    \red{For the second row, for any $\mathbf{x}\in\mathcal M$, by the triangle inequality,}
		\begin{align*}
			&\red{\|\widetilde{{\mathscr{E}}}(\mathbf{F}(\mathbf{x}))-\widetilde{\mathscr{E}}(\mathbf{F}(\mathbf{x})+\bfw)\|_2 }\\
            \red{\le}&\ \red{ \| \widetilde{\mathscr{E}}(\mathbf{F}(\mathbf{x})) - \varphi(\mathbf{F}(\mathbf{x}))\|_2 + \|\varphi(\mathbf{F}(\mathbf{x})) - \varphi(\mathbf{F}(\mathbf{x}) + \bfw) \|_2 } \\
           &\quad \red{ +\, \|\varphi(\mathbf{F}(\mathbf{x})+\bfw) - \widetilde{\mathscr{E}}(\mathbf{F}(\mathbf{x})+\bfw)\|_2} \\
			\red{\le}&\  \red{2\sqrt{n} \varepsilon_1+C_{\varphi} \|\bfw\|_2 ~.}
		\end{align*}
		Choosing $\varepsilon_1 = \min\{\frac{\varepsilon}{1+C_\phi\rho\sqrt n},\ \frac{\varepsilon}{1+C_\phi \sqrt n},\ \frac{\epsilon}{(\rho+1)\sqrt{n}} \}$, it follows that
		\begin{align}\label{chi1_3_result}
			\red{C_1}\,\mathbb{E}_{\mathcal{S}}\left[\frac{1}{N}\sum_{i=1}^N\| \widehat{\mathscr{H}}(\mathbf{x}_i) - \widehat{\mathscr{E}}\left(\red{\mathbf{x}_i^+}\right)\|_2^2 \right] 
			\le\red{ C_2 \varepsilon^2+2C_1^2C_\varphi^2\E \|\bfw\|_2^2~,}
		\end{align}
\red{where constant $C_2 =   C_1^2\left(\frac{1}{C_1 - 1}+\frac{8}{(\rho+1)^2}\right) $}.

		As for the last term, simply note that $R(\widehat{\mathbf{A}})=0$ since $\mathbf{A}\in \mathcal{F}^A_{Lin}$. Thus, for the $\chi_1$, we can obtain
		\begin{align}\label{chi_1}
			\chi_1 \leq \red{(2C_1 D+C_2)\varepsilon^2 + C_1(1+2C_1C_\varphi^2)\E \|\bfw\|_2^2 }~.
		\end{align}
		
		We now bound the first contribution in $\chi_2$, defined as
		\[
		T_1 := \mathbb{E}_{\mathcal{S}}\mathbb{E}_{(\x,\x^+)\sim \mathbb{P}_{X,X^+}} \bigl\| \widehat{\mathscr{G}}(\mathbf{x}) - \red{\mathbf{x}^+} \bigr\|_2^2 - \red{C_1}\mathbb{E}_{\mathcal{S}}\!\left[ \frac{1}{N}\sum_{i=1}^N \bigl\| \widehat{\mathscr{G}}(\mathbf{x}_i) - \red{\mathbf{x}_i^+} \bigr\|_2^2\right].
		\]

        According to \cite[Lemma 2]{LIU2024DAE}, \red{the result can be extended from the case $C_1=2$ to any $C_1>1$ by a straightforward modification of the proof. In particular, replacing the constant $\frac{1}{8DB^2}$ in (66) therein with $\frac{C_1-1}{4C_1DB^2}$ yields that for any $0<\delta<1$, there exists a constant $C_3=\frac{35C_1}{2}$ such that}
        \begin{equation}\label{T2_1}
        T_1 \le \frac{\red{C_3} D B^2}{N} \log \mathcal{Q}\!\left( \frac{\delta}{\red{C_1} D B}, \mathcal{F}_{\mathrm{NN}}^{\mathscr{G}},\|\cdot\|_{L^{\infty,\infty}}\right) + \red{3C_1}\delta .
        \end{equation}
		It remains to estimate the covering number of the composite class $\mathcal{F}_{\mathrm{NN}}^{\mathscr{G}}$. Specifically, for any $\mathscr{G}_k = \mathscr{D}_k\circ \mathbf{A}_k \circ \mathscr{E}_k$, $k=1,2$, we have the uniform bound
		\begin{align*}
			&\,\|\mathscr G_1-\mathscr G_2\|_{L^{\infty,\infty}} \\
			=& \, \sup_{x\in\mathcal M}\|\mathscr{D}_1\circ \mathbf{A}_1\circ \mathscr{E}_1    - \mathscr{D}_2 \circ \mathbf{A}_2 \circ \mathscr{E}_1 \|_{L^{\infty,\infty}} \\
			\le &\, \sup_{\x\in\mathcal M}\Big\| \mathscr D_1\circ \mathbf A_1\circ \mathscr E_1(\x) -\mathscr D_2\circ \mathbf A_1 \circ  \mathscr E_1(\x)\Big\|_{\infty}
			 \\&\, + \sup_{\x\in\mathcal M}\Big\| \mathscr D_2 \circ  \mathbf A_1\circ \mathscr E_1(\x) -\mathscr D_2\circ \mathbf A_2\circ \mathscr E_2(\x)\Big\|_\infty \\
			\le&\; \|\mathscr D_1-\mathscr D_2\|_{L^{\infty,\infty}} +C_\phi\sup_{\x\in\mathcal M} \Big\|\mathbf A_1\circ\mathscr E_1(\x)-\mathbf A_2\circ\mathscr E_2(\x)\Big\|_2\\
			\le&\; \|\mathscr D_1-\mathscr D_2\|_{L^{\infty,\infty}} + C_\phi\sup_{x\in\mathcal M} \Big(  \|\mathbf A_1\|_2\,\|\mathscr E_1(\x)-\mathscr E_2(\x)\|_2 +\|\mathbf A_1-\mathbf A_2\|_2\,\|\mathscr E_2(\x)\|_2 \Big)\\
			\le&\; \|\mathscr D_1-\mathscr D_2\|_{L^{\infty,\infty}} + C_\phi\rho\sqrt n\,\|\mathscr E_1-\mathscr E_2\|_{L^{\infty,\infty}} + C_\phi n^{3/2}\Lambda\,\|\mathbf A_1-\mathbf A_2\|_{\infty,\infty}.
		\end{align*}
  Thus, we choose 
		\[
		\delta_{\mathscr{D}} := \frac{\delta}{3},\qquad \delta_{\mathscr{E}} := \frac{\delta}{3 C_\phi \rho \sqrt{n}},\qquad \delta_A := \frac{\delta}{3 C_\phi n^{3/2} \Lambda},
		\]
		so that a $\delta$-cover of $\mathcal{F}_{\mathrm{NN}}^{\mathscr{G}}$ can be constructed by combining covers of the individual components,
		\begin{align*}
		&\mathcal{Q}\!\left( \delta, \mathcal{F}_{\mathrm{NN}}^{\mathscr{G}}, \| \cdot\|_{L^{\infty,\infty}} \right)\\
		 \le&\, \mathcal{Q}\!\left( \delta_{\mathscr{E}}, \mathcal{F}_{\mathrm{NN}}^{\mathscr{E}}, \|\cdot\|_{L^{\infty,\infty}} \right) \cdot \mathcal{Q}\!\left( \delta_{\mathscr{D}}, \mathcal{F}_{\mathrm{NN}}^{\mathscr{D}}, \|\cdot\|_{L^{\infty,\infty}} \right) \cdot \mathcal{Q}\!\left( \delta_A, \mathcal{F}_{\mathrm{Lin}}^{A}, \|\cdot\|_{{\infty,\infty}} \right).
		\end{align*}
		
		It remains to bound the covering number of the linear class $\mathcal{F}_{\mathrm{Lin}}^{A}$ under the $\|\cdot\|_{\infty,\infty}$ norm. Since $\|A\|_{\infty,\infty}\le \|A\|_2$, we have $\mathcal{F}_{\mathrm{Lin}}^{A}\subset\{A:\|A\|_{\infty,\infty}\le \rho\}$. The covering number of the linear class $\mathcal{F}_{\mathrm{Lin}}^{A}$ under the $\|\cdot\|_{\infty,\infty}$ norm follows from standard entropy estimates for finite-dimensional $\ell_\infty$-balls \cite{Vaart1996cover}. Hence, for any  $0\le \delta_{A}\le \rho$,
		\begin{equation}\label{cover_A}
			\log\mathcal{Q}\!\left( \delta_A,\mathcal{F}_{\mathrm{Lin}}^{A},\|\cdot\|_{\infty,\infty} \right) \le \mathcal{O}\!\left( n^2\bigl(\log\rho + \log\delta_A^{-1}\bigr)\right).
		\end{equation}
		Applying Lemma~\ref{lemma_nncover} to the encoder and decoder classes and \eqref{cover_A} to the linear class yields
		\begin{align*}
			\log \mathcal{Q}\!\left( \frac{\delta}{\red{C_1}DB}, \mathcal{F}_{\mathrm{NN}}^{\mathscr{G}},\| \cdot\|_{L^{\infty,\infty}} \right) \le\; &\mathcal{O}\!\left( L_{\mathscr{E}}K_{\mathscr{E}} \bigl( \log p_{\mathscr{E}}+\log \kappa_{\mathscr{E}}+\log\delta_{\mathscr{E}}^{-1} \bigr) \right) \\
			&+ \mathcal{O}\!\left( L_{\mathscr{D}}K_{\mathscr{D}} \bigl(\log p_{\mathscr{D}}+\log \kappa_{\mathscr{D}}+\log\delta_{\mathscr{D}}^{-1} \bigr)\right) \\
			& + \mathcal{O}\!\left( n^2\bigl(\log\rho+\log\delta_A^{-1}\bigr) \right).
		\end{align*}
		Lemma ~\ref{lem_EAD_approx} also satisfies the architecture scaling in \cite[Lemma 6 and Lemma 8]{LIU2024DAE}, we finally obtain
		\begin{align*}
			&\log \mathcal{Q}\!\left( \frac{\delta}{\red{C_1}DB}, \mathcal{F}_{\mathrm{NN}}^{\mathscr{G}}, \| \cdot\|_{L^{\infty,\infty}} \right) \\
			\le\; & \mathcal{O}\!\Bigl( D\log^2 D\, \varepsilon^{-n}\log^4\varepsilon^{-1} + D\log^2 D \, \varepsilon^{-n}\log^3\varepsilon^{-1}\log\delta^{-1} + n^2\bigl(\log\rho+\log\delta^{-1}\bigr) \Bigr).
		\end{align*}
		Setting $\delta=\varepsilon=N^{-\frac{1}{n+2}}$, we obtain
		\begin{align}\label{T2_11}
			T_1 \le\mathcal{O}\!\left(D^2(\log^2 D)\,N^{-\frac{2}{n+2}}\log^4 N\right).
		\end{align}
		We next bound the remaining two contributions in $\chi_2$. Define 
		\begin{align*}
		T_2 := &\mathbb{E}_{\mathcal{S}}\mathbb{E}_{(\x,\x^+)\sim \mathbb{P}_{X,X^+}} \bigl\|\widehat{\mathscr D}\circ\widehat{\mathscr E}(\x)-\x\bigr\|_2^2 - \red{C_1}\mathbb{E}_{\mathcal{S}}\!\left[ \frac{1}{N}\sum_{i=1}^N \bigl\|\widehat{\mathscr D}\circ\widehat{\mathscr E}(\mathbf x_i)-\mathbf x_i\bigr\|_2^2 \right],\\
		T_3 :=& \mathbb{E}_{\mathcal{S}}\mathbb{E}_{(x,x^+)\sim \mathbb{P}_{X,X^+}} \bigl\|\widehat{\mathscr H}(\mathbf x)-\widehat{\mathscr E}( \red{\mathbf x^+})\bigr\|_2^2 - \red{C_1} \mathbb{E}_{\mathcal{S}}\!\left[ \frac{1}{N}\sum_{i=1}^N \bigl\|\widehat{\mathscr H}(\mathbf x_i)-\widehat{\mathscr E}(\red{\mathbf x^+_i})\bigr\|_2^2 \right].
		\end{align*}
		Similarly, \red{for any $0<\delta<1$,  we obtain the following}
		\begin{align}
			T_2 \le& \, \frac{\red{C_3} D B^2}{N}\log \mathcal{Q}\!\left(\frac{\delta}{\red{C_1} DB}, \mathcal{F}_{\mathrm{NN}}^{\mathscr{DE}},\|\cdot\|_{L^{\infty,\infty}}\right) +\red{3C_1}\delta \notag\\
            \leq& \, \mathcal{O}\!\left( D^2(\log^2 D)\,N^{-\frac{2}{n+2}}\log^4 N \right),\label{T2_2}\\ 
			T_3 \le& \, \frac{\red{C_3} n \Lambda^2}{N}\log \mathcal{Q}\!\left(\frac{\delta}{\red{C_1} n\Lambda}, \mathcal{F}_{\mathrm{AE}}^{\mathscr{H}},\|\cdot\|_{L^{\infty,\infty}}\right)  + \red{ 3C_1}\delta\notag\\
            \leq&\, \mathcal{O}\!\left( nD(\log^2D)N^{-\frac{2}{n+2}}\log^4 N \right).\label{T2_3}
		\end{align}
		Combining \eqref{T2_11}--\eqref{T2_3} with the above covering-number bounds, we conclude that
		\[
		\chi_2 = T_1+T_2+T_3 \le \mathcal{O}\!\left( D^2(\log^2 D)\,N^{-\frac{2}{n+2}}\log^4 N \right).
		\]
		
		Consequently,
		\begin{align*}
			&\mathbb{E}_{\mathcal{S}}\mathbb{E}_{(\x,\x^+)\sim \mathbb{P}_{X,X^+}} \Bigl[ \| \widehat{\mathscr{G}}(\mathbf{x}) -\red{\mathbf{x}^+}\|_2^2 + \| \widehat{\mathscr{D}}\circ\widehat{\mathscr{E}}(\mathbf{x}) - \mathbf{x}\|_2^2 + \| \widehat{\mathscr{H}}(\mathbf{x}) - \widehat{\mathscr{E}}(\red{\mathbf{x}^+})\|_2^2 \red{+ R(\widehat{\mathbf{A}})} \Bigr] \\
			\leq\;&  \red{(2C_1 D+C_2)N^{-\frac{2}{n+2}} + C_1(1+2C_1C_\varphi^2)\E \|\bfw\|_2^2 }   +   C_0 \,D^2(\log^2 D)\,N^{-\frac{2}{n+2}}\log^4 N \\
			\leq\;& \red{C_1(1+2C_1C_\varphi^2)\E \|\bfw\|_2^2} + \red{C_0}\,D^2(\log^2 D)\,N^{-\frac{2}{n+2}}\log^4 N,
		\end{align*}
		for some constant $\red{C_0}$ depending on $n, C_1, C_\varphi, C_\phi, \Lambda, B, \rho$, and the volume of $\mathcal{M}$, 	which completes the proof.
	\end{proof}
{\cred{
\section{Proof of Theorem~\ref{thm2}}\label{appenB}
\begin{proof}
Let $\mathcal F_{\mathrm{NN}}^{obs}(D_y,m;L_{obs},p_{obs},K_{obs},\kappa_{obs},R_{obs})$ denote the observation encoder class. We follow the same decomposition as in Theorem~\ref{thm1}:
\begin{align*}
\chi_1^{obs} &= C_1 \mathbb{E}_{\mathcal S} \left[ \frac{1}{N}\sum_{i=1}^N \| \widehat{\mathscr E}_{obs}(\y_i) - \mH\,\mathscr E(\x_i) \|_2^2 \right],\\
\chi_2^{obs} &= \mathbb{E}_{\mathcal S}\mathbb{E}_{(\x,\y)} \| \widehat{\mathscr E}_{obs}(\y) - \mH\,\mathscr E(\x) \|_2^2 - C_1 \mathbb{E}_{\mathcal S} \left[ \frac{1}{N}\sum_{i=1}^N \| \widehat{\mathscr E}_{obs}(\y_i) - \mH\,\mathscr E(\x_i) \|_2^2 \right].
\end{align*}
By optimality of $\widehat{\mathscr E}_{obs}$, 
\begin{align*}
\chi_1^{obs} \le C_1 \inf_{\mathscr E_{obs}\in \mathcal F_{\mathrm{NN}}^{obs}} \mathbb E_{(\x,\y)} \| \mathscr E_{obs}(\y) - \mH \mathscr E(\x) \|_2^2.
\end{align*}

By the approximation assumption on $\mathcal F_{\mathrm{NN}}^{obs}$, there exists
$\widetilde{\mathscr E}_{obs} \in \mathcal F_{\mathrm{NN}}^{obs}$ such that
\begin{align*}
\mathbb E_{(\x,\y)} \| \widetilde{\mathscr E}_{obs}(\y) - \mathbb E[\mH \mathscr E(\x)\mid \y] \|_2^2 \le \varepsilon^2.
\end{align*}
Using the standard bias–variance decomposition
\begin{align*}
\|\widetilde{\mathscr E}_{obs}(\mathbf{y}) - \mathbf{H}\,\mathscr E(\mathbf{x})\|_2^2 
\le\; 2 \|\widetilde{\mathscr E}_{obs}(\mathbf{y}) - \mathbb E[\mathbf{H}\,\mathscr E(\mathbf{x}) \mid \mathbf{y}] \|_2^2 
+ 2 \|\mathbb E[\mathbf{H}\,\mathscr E(\mathbf{x}) \mid \mathbf{y}] - \mathbf{H}\,\mathscr E(\mathbf{x})\|_2^2.
\end{align*}
Under a stability assumption on the observation operator $\mathcal H$, the conditional variance term can be controlled by the observation noise level. Thus, there exists a constant $C_1^{obs}= 2C_1 \|\mathbf H\|_2^2 C_\phi^2 C_{\mathcal H}^2$, such that
\begin{align*}
\chi_1^{obs} \le 2C_1\varepsilon^2 + C_1^{obs}\,\mathbb E\|\boldsymbol{\eta}\|_2^2 .
\end{align*}

The term $\chi_2^{obs}$ is controlled using the same symmetrization and covering number arguments as in Theorem~\ref{thm1}, giving 
\begin{align*}
\chi_2^{obs} \le \frac{C_3 m R_{obs}^2}{N} \log \mathcal Q\left( \frac{\delta}{C_1 m R_{obs}}, \mathcal F_{\mathrm{NN}}^{obs}, \| \cdot\|_{L^{\infty,\infty}}  \right) + 3C_1\delta.
\end{align*}
Using the covering number estimate for neural networks and choosing $\delta = \varepsilon$, we obtain
\begin{align*}
\chi_2^{obs} \le \mathcal O\Big( D_y \log^2 D_y \,\varepsilon^{-n}\log^4 \varepsilon^{-1} \Big),
\end{align*}
and if $\mH$ is learned, the hypothesis class is augmented by a linear operator in $\mathbb{R}^{m\times n}$, whose covering number contributes an additional $\mathcal O(mn \log \delta^{-1})$ term, affecting only the constants.


Combining the bounds yields
\begin{align*}
\mathbb{E}_{\mathcal S}\mathbb{E}_{(\x,\y)} \| \widehat{\mathscr E}_{obs}(\y) - \mH\,\mathscr E(\x) \|_2^2 \le C_0^{obs} D_y^2 \log^2 D_y N^{-\frac{2}{n+2}}\log^4 N + C_1^{obs}\,\mathbb E\|\tilde{\bfeta}\|_2^2,
\end{align*}
for some constant $C_0^{obs}>0$, which completes the proof.
\end{proof}
}}

{\cred{
\section{Algorithmic Details of Baseline Methods}
\begin{algorithm}[H]
\caption{\red{AE-EnKF}}
\cred
\label{alg:AE-EnKF}
\begin{algorithmic}[1]
\Require Encoder $\mathscr{E}$, decoder $\mathscr{D}$, initial ensemble $\{\widehat{\mathbf{x}}_0^{(j)}\}_{j=1}^{N_e}$
\For{$k = 1,\dots,K$}
\State \textbf{Forecast:} $\widehat{\mathbf{x}}_{k|k-1}^{(j)} = \mathscr{D}(\mathscr{E}(\widehat{\mathbf{x}}_{k-1}^{(j)}))$
\State \textbf{Analysis:}
Update $\widehat{\mathbf{x}}_{k|k-1}^{(j)}$ using EnKF with observation $\mathbf{y}_k$ and observation operator $\mathcal{H}$
\EndFor
\end{algorithmic}
\end{algorithm}
\begin{algorithm}[H]
\caption{\red{DAE-EnKF}}
\label{alg:DAE-EnKF}
\cred
\begin{algorithmic}[1]
\Require Encoder $\mathscr{E}$, decoder $\mathscr{D}$, latent model $\mathscr{A}$, initial ensemble $\{\widehat{\mathbf{x}}_0^{(j)}\}_{j=1}^{N_e}$
\State Encode: $\widehat{\mathbf{z}}_{0}^{(j)} = \mathscr{E}(\widehat{\mathbf{x}}_{0}^{(j)})$
\For{$k = 1,\dots,K$}
\State \textbf{Forecast:} $\widehat{\mathbf{z}}_{k|k-1}^{(j)} = \mathscr{A}(\widehat{\mathbf{z}}_{k-1}^{(j)})$
\State \textbf{Analysis:}
Update $\widehat{\mathbf{z}}_{k|k-1}^{(j)}$ using EnKF with observation $\mathbf{y}_k$ and observation operator $\mathcal{H} \circ \mathscr{D}$
\EndFor
\State Decode: $\widehat{\mathbf{x}}_k^{(j)} = \mathscr{D}(\widehat{\mathbf{z}}_k^{(j)})$
\end{algorithmic}
\end{algorithm}
}}

\end{document}